\documentclass{article}

\usepackage{microtype}
\usepackage{graphicx}
\usepackage{subfigure}
\usepackage{booktabs} 

\usepackage{hyperref}

\usepackage[accepted]{icml2024}

\makeatletter

\let\algorithmicrequire\relax
\let\algorithmicensure\relax
\let\algorithmiccomment\relax
\let\algorithmicend\relax
\let\algorithmicif\relax
\let\algorithmicthen\relax
\let\algorithmicelse\relax
\let\algorithmicfor\relax
\let\algorithmicforall\relax
\let\algorithmicdo\relax
\let\algorithmicwhile\relax
\let\algorithmicloop\relax
\let\algorithmicrepeat\relax
\let\algorithmicuntil\relax
\let\algorithmicfunction\relax
\makeatother

\usepackage{algorithmicx}
\usepackage{algpseudocode}

\usepackage{float}
\floatstyle{ruled}
\newfloat{algorithm}{tbp}{loa}
\floatname{algorithm}{Algorithm}

\usepackage{amsmath}
\usepackage{amssymb}
\usepackage{mathtools}
\usepackage{amsthm}
\usepackage{comment}
\usepackage{multicol}
\usepackage[capitalize,noabbrev]{cleveref}

\theoremstyle{plain}
\newtheorem{theorem}{Theorem}

\newtheorem{lemma}{Lemma}
\newtheorem{cond}{Assumption}
\newtheorem{corollary}{Corollary}
\theoremstyle{definition}
\newtheorem{definition}{Definition}

\theoremstyle{remark}
\newtheorem{remark}{Remark}
\newtheorem{example}{Example}[section]

\newcommand{\Mean}{{\mathbb{E}}}
\newcommand{\Var}{{\mbox{Var}}}
\newcommand{\Cov}{{\mbox{Cov}}}

\newcommand{\prob}{{\mathbb{P}}}
\DeclareMathOperator*{\argmin}{arg\,min}

\newcommand{\widebar}[1]{{\overline{#1}}}

\usepackage[textsize=tiny]{todonotes}

\icmltitlerunning{Combining Experimental and Historical Data for Policy Evaluation}

\begin{document}

\twocolumn[
\icmltitle{Combining Experimental and Historical Data for Policy Evaluation}



\icmlsetsymbol{equal}{*}

\begin{icmlauthorlist}
\icmlauthor{Ting Li}{equal,sufe}
\icmlauthor{Chengchun Shi}{equal,Lse}
\icmlauthor{Qianglin Wen}{yunnanU}
\icmlauthor{Yang Sui}{sufe}
\icmlauthor{Yongli Qin}{didi}
\icmlauthor{Chunbo Lai}{didi}
\icmlauthor{Hongtu Zhu}{unc}
\end{icmlauthorlist}

\icmlaffiliation{sufe}{School of Statistics and Management, Shanghai University of Finance and Economics}
\icmlaffiliation{Lse}{Department of Statistics, London School of Economics and Political Science}
\icmlaffiliation{yunnanU}{Yunnan Key Laboratory of Statistical Modeling and Data Analysis, Yunnan University}
\icmlaffiliation{didi}{Didi Chuxing}
\icmlaffiliation{unc}{Department of Biostatistics, The University of North Carolina at Chapel Hill}

\icmlcorrespondingauthor{Hongtu Zhu}{htzhu@email.unc.edu}

\icmlkeywords{Off-policy Evaluation, Data Integration, Causal Inference, Reinforcement Learning}

\vskip 0.3in
]



\printAffiliationsAndNotice{\icmlEqualContribution} 

\begin{abstract}
 This paper studies policy evaluation with multiple data sources, especially in scenarios that involve one experimental dataset with two arms, complemented by a historical dataset generated under a single control arm.
We propose novel data integration methods
that linearly integrate base policy value estimators constructed based on the experimental and historical data, with weights optimized to minimize the mean square error (MSE) of the resulting combined estimator. We further apply the pessimistic principle to obtain more robust estimators, and extend these developments to sequential decision making. 
Theoretically, we establish non-asymptotic error bounds for the MSEs of our proposed estimators, and derive their oracle, efficiency and robustness properties across a broad spectrum of reward shift scenarios. Numerical experiments and real-data-based analyses from a ridesharing company demonstrate the superior performance of the proposed estimators.
\end{abstract}

\section{Introduction}
{\bf Motivation}. This paper seeks to establish data-driven approaches for evaluating the effectiveness of a newly target policy against a conventional control. A basic approach relies solely on experimental data to formulate the treatment effect estimator, which we refer to as the experimental-data-only (EDO) estimator. However, the often limited sample size of experimental data prompts the need to 
incorporate 
auxiliary external datasets to enhance the precision of the treatment or policy effect estimator. We provide three illustrative examples to demonstrate this concept.

\textbf{Example 1: A/B testing with historical data}. A/B testing is frequently used in modern technology companies such as Amazon, eBay, Facebook, Google, LinkedIn, Microsoft, Netflix, Uber and Didi
for comparing new products/strategies against existing ones \citep[see][for a recent review]{larsen2023statistical}.
A common challenge in A/B testing is the limited experiment duration coupled with weak treatment effects. For example, in the ridesharing industry, experiments usually last no more than two weeks, and the effect sizes often range from 0.5\% to 5\% \citep{xu2018large,tang2019deep,zhou2021graph,qin2022reinforcement}. Yet, prior to these experiments, companies often have access to a substantial volume of historical data under the current policy. Leveraging such historical data can significantly improve the efficiency of A/B testing.


\textbf{Example 2: Meta analysis}.  
In medicine, data often spans across multiple healthcare institutions. A notable example 
is in the schizophrenia study that examined the efficacy of cognitive-behavioral therapy in early-stage schizophrenia patients \citep{tarrier2004cognitive}. This study was a multicenter randomized controlled trial (RCT) executed in three treatment centers (Manchester, Liverpool, and North Nottinghamshire). Given that each center had a relatively small cohort, with fewer than 100 participants, pooling data from all sources becomes crucial to enhance causal learning.

\textbf{Example 3: Combining observational data}. RCTs are widely regarded as the benchmark for learning the causal impacts of interventions or treatments on specific outcomes \citep{imbens2015causal}. However, RCTs often face practical challenges such as high costs and time constraints, resulting in limited participant numbers. Conversely, observational data, derived from sources such as biobanks or electronic health records, boast larger sample sizes.  Integrating both data presents a unique opportunity to improve the statistical learning efficiency \citep{colnet2020causal}.

\textbf{Challenge}.   The challenge of merging multiple data sources often lies in the distributional shifts that occur between them. In the domain of ridesharing and healthcare, datasets from different time frames frequently display temporal non-stationarity \citep{wan2021pattern,li2022testing,wang2023robust}. 
In the 
schizophrenia 
study, 
variations across various treatment centers and ethnic groups introduce heterogeneity, leading to distributional shifts \citep{dunn2007modelling,shi2018maximin}. Furthermore, the integration of RCTs with observational data introduces the potential for unmeasured confounding within the observational data 
\citep{10.1214/09-SS057}. Neglecting these distributional shifts would produce biased estimations of treatment effects.

\textbf{Contributions}. This paper focuses on the application of A/B testing with historical data. However, the methodologies and theoretical frameworks we develop are equally applicable to the other two examples discussed earlier. Our contributions are summarized as follows:

{\bf Methodologically,} we propose several weighted estimators for data integration, including both pessimistic and non-pessimistic estimators, covering both non-dynamic settings (also referred to as contextual bandits in the OPE literature) and sequential decision making. 
We demonstrate the superior empirical performance of these estimators through simulations and real-data-based analyses 
\footnote{R code implementing the proposed weighted estimators is available at
\url{https://github.com/tingstat/Data_Combination}.}.



{\bf Theoretically,} we derive various statistical properties (e.g., efficiency, robustness and oracle property) of the proposed estimators across a wide range of scenarios, accommodating varying degrees of reward shift between the experimental data and the historical data in mean -- from zero to small (the shift's order is much smaller than $n^{-1/2}$ ), moderate (the 
shift's order falls between $n^{-1/2}$ and $n^{-1/2} \sqrt{\log(n)}$), and large (the shift is substantially larger than $n^{-1/2}\sqrt{\log(n)}$), where $n$ represents the effective sample size; see Table \ref{table:summary_theoretical_results} for a summary. To the contrary, existing works impose more restrictive conditions. They either require the mean shift to be zero, or sufficiently large for clear detection \citep[see e.g.,][]{cheng2021adaptive,han2021federated,dahabreh2023efficient,li2023improving}. 
In summary, our findings suggest that the non-pessimistic estimator tends to be effective in scenarios where the reward shift is minimal or substantial. In contrast, the pessimistic estimator demonstrates greater robustness, particularly in situations with moderately large reward shifts.

\begin{table}
\begin{center}
\small 
\caption{Summary of the properties of MSEs of our estimators.}
\label{table:summary_theoretical_results}
\tabcolsep 2pt
\begin{tabular}{l|l|l}
	\hline
	\hline
Reward shift  &   Non-pessimistic estimator &  Pessimistic estimator \\
 \hline
 Zero    & Close to efficiency bound & Same order to oracle MSE \\ 
Small    &   Close to oracle MSE  & Same order to oracle MSE            \\
Moderate &   May suffer a large MSE         & Oracle property \\
Large    & Oracle property    &  Oracle property \\
	\hline \hline     
\end{tabular}
\end{center}
\footnotesize{The {\bf oracle} MSE denotes MSE of the oracle estimator that use the {\bf best} weight to combine historical and experimental data whereas the {\bf efficiency bound} is the smallest achievable MSE among a broad class of regular estimators \citep{tsiatis2006semiparametric}.} 
\vspace{-0.25in}
\end{table}

\section{Related Work}
{\bf Data integration in causal inference.}
There is a growing literature on combining randomized data with other sources of datasets; see \citet{degtiar2023review} and \citet{shi2023data} for reviews. These methods can be broadly classified into three categories, as outlined in the introduction:\vspace{-0.5em}
\begin{enumerate}
    \item 
    The first category leverages historical datasets collected under the control \citep[see e.g.,][]{pocock1976combination,cuffe2011inclusion,viele2014use,van2018including,schmidli2020beyond,cheng2023enhancing,liu2023causal,scott2024borrowing}. In particular, assuming no reward shift, \citet{li2023improving} developed a semi-parametric efficient estimator whose MSE achieves the efficiency bound. In contrast, our methods are more flexible, allowing the reward shift to exist.\vspace{-0.5em}
    \item The second category is meta analysis where the external data is collected from different trials \citep{schmidli2014robust,dersimonian2015meta, hasegawa2017myth, zhang2019bayesian,steele2020importance,lian2023accounting,rott2024causally}.
    This category includes a notable subset of methods that apply $\ell_1$-type penalty functions for selecting external data \citep{dahabreh2020towards,han2021federated,han2023multiply}. However, their performance is sensitive to the choice of the tuning parameter, as shown in our numerical study (see  Figure \ref{fig:real_data_based_Lasso_all}).
    \vspace{-0.5em}
\item The last category incorporates observational data to enhance causal learning \citep[see e.g.,][]{hartman2015sate,peysakhovich2016combining,kallus2018removing, athey2020combining,gui2020combining,yang2020combining,yang2020improved,wu2022integrative,lee2023improving}.
\end{enumerate}
\begin{remark}
The first category of research is closely related to our work, while the focus of the last two categories differs from ours. Additionally, all aforementioned studies concentrate on the non-dynamic setting framework. Our research, however, broadens this perspective by investigating sequential decision making where treatments are assigned sequentially over time -- a typical scenario studied in reinforcement learning \citep[RL,][]{sutton2018reinforcement}.
\end{remark}


{\bf Offline policy learning.}  
The proposed pessimistic estimator is inspired by recent advancements in offline policy learning, which aims to learn an optimal policy from a pre-collected offline dataset without active exploration of the environment. 
Existing methods typically adopt the ``pessimistic principle" to mitigate the discrepancy between the behavior policy that generates the offline data and the optimal policy. In contrast to the optimistic principle widely used in contextual bandits and online RL, the pessimistic principle favors actions whose values are less uncertain.

In non-dynamic settings, pessimistic algorithms can generally be categorized into value-based and policy-based methods. Value-based methods learn a conservative reward function to prevent overestimation and compute the greedy policy with respect to this estimated reward function \citep{buckman2020importance,jin2021pessimism, rashidinejad2021bridging,zhou2023optimizing}. Conversely, policy-based methods directly search the optimal policy by either restricting the policy class to stay close to the behavior policy or optimizing the policy that maximizes an estimated lower bound of the reward function \citep{swaminathan2015batch,swaminathan2015self,wu2018variance,kennedy2019nonparametric,aminian2022semi,jin2022policy,zhao2023positivity}.

Furthermore, the pessimistic principle has been extensively adopted in offline RL, accommodating more complex sequential settings \citep{kumar2019stabilizing,kumar2020conservative,yu2020mopo,uehara2021pessimistic,xie2021bellman,bai2022pessimistic,rigter2022rambo,shi2022pessimistic,yin2022near,zhou2023bi}.

{\bf Off-policy evaluation (OPE).}
Finally, our work is closely related to OPE in contextual bandits and RL, which aims to estimate the mean outcome of a new target policy using data collected by a different policy \citep[see][for reviews]{dudik2014doubly,uehara2022review}. It has been recently employed to conduct A/B testing in sequential decision making \citep{bojinov2019time,farias2022markovian,tang2022reinforcement,shi2023multiagent,shi2023dynamic,li2024evaluating,wen2024analysis}. 
Existing approaches in this field can generally be classified into three main groups:\vspace{-0.5em}
\begin{enumerate}
    \item \textbf{Direct methods}: these methods learn a reward or value function from offline data to estimate the policy value \citep{bradtke1996linear,le2019batch,feng2020accountable,luckett2020estimating,hao2021bootstrapping,liao2021off,chen2022well,shi2022statistical,bian2023off,uehara2024future}.\vspace{-0.5em}
    \item \textbf{Importance sampling (IS) methods}: this group employs the IS ratio to adjust the observed rewards, accounting for the discrepancy between the target policy and the behavior policy \citep{heckman1998matching,hirano2003efficient,thomas2015high,liu2018breaking,dai2020coindice,wang2023projected,hu2023off}.\vspace{-0.5em}
    \item \textbf{Doubly robust (DR) methods}: these strategies integrate the principles of direct methods and IS  \citep{tan2010bounded,dudik2011doubly,van2011cross,zhang2012robust,jiang2016doubly,thomas2016data,chernozhukov2018double,farajtabar2018more,oprescu2019orthogonal,shi2020breaking,uehara2020minimax,kallus2022efficiently,liao2022batch}. Their validity relies on the consistency of either the direct method or IS, but not necessarily both. We refer to such a property as the double robustness property. \vspace{-0.5em}
\end{enumerate}
We note that none of the aforementioned work studied data integration, which is the central theme of this paper.

\section{Estimators in Non-dynamic Setting}\label{sec:contextagnosticest}
{\bf Summary.}  In this section, we present our newly developed non-pessimistic and pessimistic estimators, tailored for non-dynamic settings. These estimators differ in their approach to weighting historical and experimental data. The non-pessimistic estimator determines its weight by minimizing an estimated MSE, whereas the pessimistic estimator optimizes its weight to minimize a ``pessimistic'' version of the estimated MSE.  We will discuss adaptations of these estimators for sequential decision making later.

{\bf Data.} The offline data comprises an experimental dataset $\mathcal{D}_e$ and a historical dataset $\mathcal{D}_h$. In the experimental setting, the decision maker observes certain contextual information at each time point, denoted by $S_e$, and makes a choice, represented by $A_e$,  between a baseline control policy $A_e=0$ and a target policy $A_e=1$, resulting in an immediate reward, $R_e$. Thus, the experimental data contains a set of i.i.d. context-action-reward triplets. In contrast, the historical data consists of i.i.d. context-reward pairs $O_h=(S_h, R_h)$ generated solely under the control policy.

{\bf Objective.} Our objective is to estimate the difference between the mean outcome under the target policy and that under the control in the experimental data. This estimand is commonly referred to as the average treatment effect (ATE). Since no unmeasured confounders exists during the experiment, the ATE can be represented by $\tau_e=\Mean [r_e^*(1,S_e) - r_e^*(0,S_e)]$,  where $r_e^*(a,s)=\Mean (R_e|A_e=a,S_e=s)$.

{\bf Two base estimators}. We next introduce two base estimators for ATE. The first estimator is the EDO estimator $\widehat{\tau}_e$, which exclusively uses $\mathcal{D}_e$ to learn ATE. The second estimator $\widehat{\tau}_h$, on the other hand, incorporates $\mathcal{D}_h$ into the ATE estimation. Specifically, it uses $\mathcal{D}_e$ to estimate the target policy's value and $\mathcal{D}_h$ to estimate the control policy's value. 

Mathematically, let $O_e$ be a shorthand for the triplet $(S_e,A_e,R_e)$ in the experimental data. We define an estimation function $\psi_e(\bullet)$ for $O_e$ as follows
\begin{eqnarray*}
    \sum_{a=0}^1 (-1)^{a-1} \Big\{r_e(a,S_e)+\nu^a(A_e|S_e) [R_e-r_e(A_e,S_e)]\Big\},
\end{eqnarray*}
where $r_e$ represents our posit model for the reward function $r_e^*$. Moreover, $\nu^a$ denotes the model for the IS ratio $\mathbb{I}(A_e=a)/\mathbb{P}(A_e=a|S_e)$, where $\mathbb{I}(\bullet)$ is the indicator function and the denominator is the behavior policy (or the propensity score) that generates $\mathcal{D}_e$. 
The following definition gives the EDO estimator.
\begin{definition}[{\bf Experimental-data-only Estimator}]
The doubly robust estimator based on the experimental data alone is defined as 
\begin{eqnarray}\label{eqn:taue}
    \widehat{\tau}_e=\frac{1}{|\mathcal{D}_e|}\sum_{O_e\in \mathcal{D}_e}\psi_e(O_e).
\end{eqnarray}
\end{definition}


\begin{remark}
    $\widehat{\tau}_e$ covers a range of estimators including the direct method estimator, the IS estimator, and the DR estimator. In its general form, without any specific restrictions, $\widehat{\tau}_e$ functions as the doubly robust estimator. It can be easily verified that $\widehat{\tau}_e$ is consistent to ATE when either the reward function or the IS ratio is correctly specified. Meanwhile, 
    $\widehat{\tau}_e$ can be simplified to the direct method estimator by setting $\nu^0=\nu^1=0$, and to the IS estimator when  $r_e$ is set to $0$.
\end{remark}
Define $r_h^*(\bullet)=\Mean (R_h|S_h=\bullet)$ as the reward function in the historical data and let $\mu^*(\bullet)$ denote the density ratio of the probability mass/density function of $S_e$ over that of $S_h$. 
The historical data distribution might differ from the experimental data distribution in the following two aspects:\vspace{-0.5em}
\begin{enumerate}
    \item \textbf{Reward shift}: the reward function $r_h^*(\bullet)$ might differ from $r_e^*(0,\bullet)$ conditional on the control.\vspace{-0.5em}
    \item \textbf{Covariate shift}: the distribution of $S_e$ might differ from that of $S_h$, i.e., $\mu^*\neq 1$. \vspace{-0.5em}
\end{enumerate}
Let $r_h$ and $\mu$ denote the posit models for $r_h^*$ and $\mu^*$. 
We next give the definition of historical-data-based estimator  $\widehat{\tau}_h$
\begin{definition}[{\bf Historical-data-based Estimator}]
   The doubly robust estimator that uses the historical data to estimate the control policy's value is defined as
   \begin{eqnarray*}
   \widehat{\tau}_h=\frac{1}{|\mathcal{D}_e|}\sum_{O_e\in \mathcal{D}_e} \psi_{h,1}(O_e) 
    -\frac{1}{|\mathcal{D}_h|}\sum_{O_h\in \mathcal{D}_h}\psi_{h,2}(O_h),
\end{eqnarray*}
where 
$\psi_{h,1}(O_e) = r_e(1,S_e)+\nu^1(A_e|S_e) [R_e-r_e(A_e,S_e)] - r_h(S_e)$
and $\psi_{h,2}(O_h) = \mu(S_h) [R_h-r_h(S_h)]$.
\end{definition}

\begin{remark}
By addition and subtraction, $\widehat{\tau}_h$ is unbiased to the difference between  $\psi_{h,1}(O_e)+r_h(S_e)$ 
and $r_h(S_e)+\psi_{h,2}(O_h)$. To elaborate $\widehat{\tau}_h$, it is crucial to understand these two estimating functions.  
The first function 
uses the experimental data to construct the doubly robust estimator for the value of the target policy, while the second function incorporates the historical data to construct the doubly robust estimator for the average outcome under the control policy. These two terms are unbiased estimators for $\mathbb{E}[r_e(1, S_e)]$ and $\mathbb{E}[r_h(S_e)]$, respectively, provided that either the density ratio 
or the reward function 
is correctly specified.



Additionally, the use of   $r_h(S_e)+\psi_{h,2}(O_h)$ partially addresses the distributional shift between the experimental and historical data. Specifically, by using $S_e$ instead of $S_h$ in $r_h(S_e)$ and using the density ratio $\mu$ in $\psi_{h,2}$, it addresses the covariate shift, leading to an unbiased estimator toward $\Mean [r_h(S_e)]$ instead of $\Mean [r_h(S_h)]$. However, it  introduces a potential bias equal to 
\begin{eqnarray}\label{eqn:bias}
    b_h=\Mean [r_e(0,S_e)]-\Mean [r_h(S_e)].
\end{eqnarray} 
This parameter represents the mean reward shift between the experimental and historical data, serving as a pivotal metric for quantifying discrepancies between the two datasets. It is also equal to the bias of the ATE estimator $\widehat \tau_h$ which incorporates the historical data to estimate the control policy's outcome. A small value of $b_h$ implies a relatively safe use of historical data to enhance the precision of the ATE estimator. Conversely, a large $b_h$ suggests caution against using historical data due to the significant bias it introduces.
\end{remark}
\vspace{-0.5em}
{\bf The proposed estimators}. Both the proposed non-pessimistic and pessimistic estimators are formulated as linear combinations of the two base estimators $\widehat{\tau}_e$ and $\widehat{\tau}_h$. 
\begin{definition}[{\bf Weighted Estimator}]
  The weighted estimator is defined as
  $$
  \widehat{\tau}_w=w \widehat{\tau}_e+(1-w) \widehat{\tau}_h
  $$
  for some properly chosen weight $w\in [0,1]$
\end{definition}
The weight is selected to minimize the MSE of the resulting estimator. Specifically, for a given $w$, according to the bias-variance decomposition, we obtain
\begin{eqnarray}\label{eqn:tauw}
    \textrm{MSE}(\widehat{\tau}_w)=\textrm{Bias}^2(\widehat{\tau}_w)+\Var(\widehat{\tau}_w),
\end{eqnarray}
where the bias is proportional to $(1-w)$, given by $-(1-w)b_h$ according to \eqref{eqn:bias} and the variance term equals 
$w^2\Var(\widehat{\tau}_e)+(1-w)^2\Var(\widehat{\tau}_h)+2w(1-w)\Cov(\widehat{\tau}_e,\widehat{\tau}_h)$.
This yields a close-form expression for \eqref{eqn:tauw}. 

We aim to estimate the oracle weight $w^*$ that minimizes \eqref{eqn:tauw}. We first note that the variance/covariance terms can be consistently estimated using the sampling variance formula\footnote{For an i.i.d. average $\bar{X}=\sum_{i=1}^n X_i/n$, its variance can be consistently estimated by $\sum_{i=1}^n (X_i-\bar{X})^2/n(n-1)$.}. It remains to estimate the reward shift bias $b_h$.

The non-pessimistic estimator employs the unbiased estimator $\widehat{b}_h = \widehat{\tau}_e - \widehat{\tau}_h$ for estimating $b_h$. Through certain derivations, this approach yields the subsequent estimator for  $w^*$: 
\begin{eqnarray}\label{eqn:non-pess}
    \widehat{w}=\frac{\widehat{b}_h^2+\widehat{\Var}(\widehat{\tau}_h)-\widehat{\Cov}(\widehat{\tau}_e,\widehat{\tau}_h)}{\widehat{\Var}(\widehat{\tau}_e)+\widehat{b}_h^2+\widehat{\Var}(\widehat{\tau}_h)-2\widehat{\Cov}(\widehat{\tau}_e,\widehat{\tau}_h)},
\end{eqnarray}
where the precise expressions for the estimated variance and covariance terms are detailed in Appendix \ref{sec::implementation}.

We next discuss the limitations of the non-pessimistic estimator. We draw a parallel with the ``offline bandit'' problem where each weight $w$ represents an arm in the bandit framework. Here, the MSE of 
$\widehat{\tau}_w$ is analogous 
to the cost of selecting an arm, with the aim being to identify the optimal arm (weight) that minimizes this cost.

For each arm, the estimated cost, $\widehat{\mathrm{MSE}}(\widehat{\tau}_w)$, is calculated by incorporating $\widehat{b}_h$ and the estimated variance/covariance values into \eqref{eqn:tauw}. The non-pessimistic estimator employs a greedy action selection method, selecting the arm with the lowest estimated cost. The reliability of this estimator is largely dependent on a uniform consistency condition, which requires that the estimated costs uniformly converge to the actual costs across all arms. However, this condition is likely to be violated, if the estimated cost for any sub-optimal arm is 
inconsistent, as depicted in Figure  \ref{fig:pessi_illustration}.

\begin{figure}
	\centering
	\includegraphics[height=3.5cm, width=8cm]{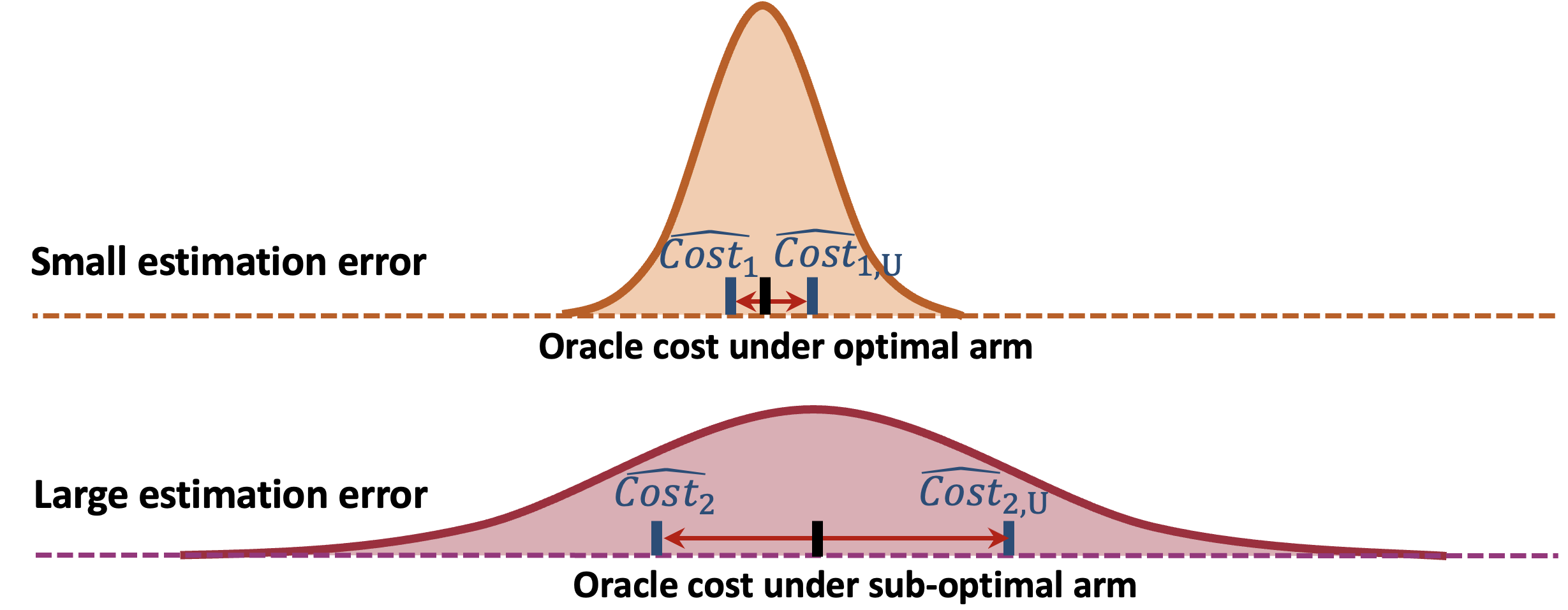}
	\vspace{-0.1 in}
	\caption{\small Distributions of estimated costs for optimal and sub-optimal arms. A key challenge arises when the estimated cost of a sub-optimal arm is inaccurately high, leading to failure of the greedy action selection method. To address this issue, we apply the pessimistic principle which takes into account the uncertainties inherent in these estimations. 
 The estimates of the cost under the two arms are given by $\widehat{Cost}_k~(k=1,2)$ with their pessimistic versions $\widehat{Cost}_{k,U} ~(k=1,2)$.
 By comparing the upper bounds of the estimated costs, we effectively identify the optimal arm.}
	\label{fig:pessi_illustration}
 \vspace{-0.2 in}
\end{figure}

In our framework, underestimating the absolute value of $b_h$ leads to lower estimated MSEs for smaller weights. As a result, the weight $\widehat{w}$ chosen by the non-pessimistic estimator tends to be smaller than the ideal (oracle) value, resulting in a significant bias in $\widehat{\tau}_{\widehat{w}}$. This 
reveals the limitations of the non-pessimistic estimator, particularly when $b_h$ is moderately large, as detailed in Table \ref{table:summary_theoretical_results} and further elaborated in Section \ref{sec:theorysampleestimator}.


The pessimistic estimator addresses this limitation by incorporating the uncertainty of cost estimation. Instead of selecting the greedy arm with the lowest estimated cost, it selects the arm based on a more pessimistic cost estimate that upper bounds the oracle cost with a high probability. This method relaxes the stringent 
uniform consistency condition, as illustrated in Figure \ref{fig:pessi_illustration}. Importantly, the consistency of the resulting estimator 
relies on the accurate estimation of the optimal arm's cost only.

In our setup, we compute an uncertainty quantifier $U$ for the estimation error $\widehat{b}_h-b_h$. It satisfies the following condition,\vspace{-0.5em}
\begin{eqnarray}\label{eqn:quantifier}
    \prob(|\widehat{b}_h-b_h|\le U)\le 1-\alpha,
\end{eqnarray}
for a given significance level $\alpha>0$. In practice, $U$ can be constructed using concentration inequalities or asymptotic normal approximation \citep{casella2021statistical}. We next use $(|\widehat{b}_h|+U)^2$ as a pessimistic estimator for $b_h^2$ and plug-in this estimator into the right-hand-side of \eqref{eqn:non-pess} to determine the weight $\widehat{w}_U$. Under the event defined in \eqref{eqn:quantifier}, $(|\widehat{b}_h|+U)^2$ serves as a valid upper bound for $b_h^2$. This leads to the pessimistic estimator $\widehat{\tau}_{\widehat{w}_U}$. We show in Section \ref{sec:pessitheory} that this estimator is more robust than the non-pessimistic estimator, particularly when $b_h$ is moderately large.

{\bf Confidence Interval.}
Under certain regularity conditions, each weighted estimator is asymptotically normal such that
$(\widehat \tau_{\widehat w} - \tau_e ) / \sqrt{ Var(\widehat \tau_{\widehat w})} \overset{d}{\rightarrow} N(0,1).$
This motivates us to consider the following Wald-type confidence interval for ATE
$$[\widehat \tau_{\widehat w}-\Phi^{-1}(1-\alpha) \sqrt{\widehat{Var}(\widehat \tau_{\widehat w})}, \widehat \tau_{\widehat w}+\Phi^{-1}(1-\alpha)\sqrt{\widehat{Var}(\widehat \tau_{\widehat w})}],$$
where $\Phi^{-1}$ is the inverse cumulative distribution function of a standard random variable, and the variance of $\widehat \tau_{\widehat w}$ is estimated based on the sampling variance formula.

\section{Extension to Sequential Decision Making}\label{sec:ext}
We next briefly outline the extension of our methods to sequential decision making. To save space, more details are given in Appendix \ref{sec:theoretical_sequential}. This extension aligns closely with our ridesharing example, where policy decisions are made sequentially over time, and past policies can influence future outcomes \citep{bojinov2023design,shi2023dynamic}. In this setting, the ATE is defined as the difference in expected cumulative rewards between the control and target policies.

The online experiment spans multiple days, with daily data summarized as sequences of state-action-reward triplets. Actions are binary, denoting either a baseline control or an experimental target policy. To account for day-to-day variations, we model the experiment as a time-varying Markov decision process. For ATE estimation, we employ the double RL estimator \citep{kallus2020double}, leading to the development of the EDO estimator. The historical data comprises state-reward pair sequences from previous days under the control policy, forming the basis for our second estimator. This estimator, also doubly robust, is used to estimate the cumulative reward under the control policy, thereby facilitating ATE calculation.  Building on the approaches outlined in Section \ref{sec:contextagnosticest}, we apply both pessimistic and non-pessimistic strategies to integrate these base estimators.

\section{Theoretical Properties}

To simplify our theoretical analysis, this section examines a sample-split version of the proposed estimator in non-dynamic settings. Further extensions of our analysis to sequential decision making 
are detailed in Appendix \ref{sec:theoretical_sequential}.

Our analysis 
compares three key estimators: a conceptual oracle estimator $\widehat{\tau}_{w^*}$, which utilizes the ideal $w^*$ value, the EDO estimator detailed in   \eqref{eqn:taue}, and the semi-parametrically efficient (SPE) estimator \citep{li2023improving} developed on the assumption of no reward shift. The EDO and SPE estimators represent two polar views of reward shift: the EDO anticipates a notable divergence between $r_e(0,\bullet)$ and $r_h(\bullet)$, while the SPE assumes no difference.


\textbf{Summary}. Before delving into the technical details, we offer a concise summary of our theories:\vspace{-0.5em}
\begin{itemize}
    \item \textbf{Small} $b_h$: in scenarios where the reward shift is much smaller than the standard deviations of the doubly robust estimators, the SPE estimator achieves the best performance. However, our analysis shows that the MSEs of the proposed estimators closely approximate those of both the oracle and the SPE estimator. \vspace{-0.5em}
    \item \textbf{Moderate} $b_h$: when the reward shift is comparable to or larger than the standard deviation terms, yet falls within the high confidence bounds of the estimation error, it remains uncertain which estimator (other than the oracle estimator) outperforms the rest. In these settings, the MSE of our pessimistic estimator is generally smaller than that of the non-pessimistic estimator.\vspace{-0.5em}
    \item \textbf{Large} $b_h$: when the reward shift is much larger than the estimation error, both the EDO estimator and our estimators are equivalent to that of the oracle estimator. We refer to this equivalence as the oracle property. 
\end{itemize}
\subsection{Properties of the Non-pessimistic Estimator}\label{sec:theorysampleestimator}
We study a sample-split variant of our estimator, where the dataset is equally divided into two parts. The first half, labeled as $\mathcal{D}_{e}^{(1)}\cup \mathcal{D}_{h}^{(1)}$, is utilized to deduce the weight $\widehat{w}$. The second half, $\mathcal{D}_{e}^{(2)}\cup \mathcal{D}_{h}^{(2)}$, is then employed to construct the final doubly robust estimator $\widehat{\tau}$, leveraging the previously estimated weight. This sample-splitting approach removes the dependencies between the estimated weight and the dataset used in formulating the ATE estimator, considerably simplifying our theoretical analysis. 
It has been widely used in causal inference and OPE \citep[see e.g.,][]{luedtke2016statistical, chernozhukov2018double, kallus2020double, bibaut2021post, shi2021deeply}. An alternative method involves swapping the roles of the data subsets $\mathcal{D}_{e}^{(1)}\cup \mathcal{D}_{h}^{(1)}$ and $\mathcal{D}_{e}^{(2)}\cup \mathcal{D}_{h}^{(2)}$ to generate a second estimator and then averaging both estimators to attain 
full efficiency. Nonetheless, this approach is not explored further in our paper for the sake of simplicity. 


We impose the following assumptions. 
\begin{cond}[Coverage]\label{con:coverage}
    Let $\pi^*(a|s)=\prob(A_e=a|S_e=s)$ be the propensity score. 
    There exists a scalar  $\epsilon>0$ such that $\pi^*(a|s)\ge \epsilon$ and $\mu^*(s)\le \epsilon^{-1}$ hold for any $a$ and $s$.  
\end{cond}

\begin{cond}[Boundedness]\label{con:bound}
    (i) There exists some constant $R_{\max}$ such that $\max(|R_e|, |R_h|)\le R_{\max}$ holds almost surely. (ii) $\max(|r_e|, |r_h|)$ is upper  bounded by $R_{\max}$. 
    (iii) $\pi$ and $\mu$ are lower bounded by $\epsilon$.  
\end{cond}

\begin{cond}[Doubly-robust Specification]\label{con:double}
    Either the reward functions or the density ratios are correctly specified. 
\end{cond}

\begin{remark}
    The coverage condition in Assumption \ref{con:coverage} is frequently imposed in the OPE literature \citep[see e.g.,][]{uehara2022review}. It is also referred to as the positivity assumption in the causal inference literature \citep{hernan2010causal}. 
\end{remark}

\begin{remark}
The condition of bounded rewards in Assumption \ref{con:bound}(i) is commonly imposed in RL \citep[see e.g.,][]{agarwal2019reinforcement}. 
Given the bounded nature of the reward and the density ratio/propensity score, it is reasonable to assume that the user-defined nuisance functions are similarly bounded, as detailed in Assumptions \ref{con:bound}(ii) and \ref{con:bound}(iii). 
\end{remark}

\begin{remark}
    Assumption \ref{con:double} reflects the double robustness property of the proposed estimator. Alternatively, this assumption can be replaced by requiring these nuisance functions to satisfy certain convergence rate conditions  \citep{chernozhukov2018double,kallus2020double}. 
\end{remark}
We begin by providing a non-asymptotic upper bound for MSE of the non-pessimistic estimator. 
Define $n_{\min} = \min\{|D_e|, |D_h|\}$ as the effective sample size. 
\begin{theorem}[MSE of the non-pessimistic estimator]\label{thm:nonpSAE}
Under Assumptions \ref{con:coverage} -- \ref{con:double}, the excess MSE of the non-pessimistic estimator compared to $\widehat{\tau}_{w^*}$, i.e., $\textrm{MSE}(\widehat{\tau}_{\widehat{w}})-\textrm{MSE}(\widehat{\tau}_{w^*})$ can be upper bounded by
\begin{eqnarray}\label{eqn:MSEdecomposition}
    \Mean [(1-w^*)^2-(1-\widehat{w})^2] (\widehat{b}_h^2-b_h^2)+O\Big(\frac{R_{\max}^2}{\epsilon^2 n_{\min}^{3/2}}\Big).
\end{eqnarray}
\end{theorem}
The upper bound can be decomposed 
into two parts: the first one represents the error for estimating the mean reward shift $b_h$, whereas the second one upper bounds the  errors for estimating the variance and covariance terms, namely $\Var(\widehat{\tau}_e)$, $\Var(\widehat{\tau}_h)$, and $\Cov(\widehat{\tau}_e,\widehat{\tau}_h)$. 

We next compare this excess MSE against $\textrm{MSE}(\widehat{\tau}_{w^*})$. First, we observe that when $\textrm{MSE}(\widehat{\tau}_{w^*})$ is proportional to $R_{\max}^2/(\epsilon^2 n_{\min})$, the second term in \eqref{eqn:MSEdecomposition} becomes negligible as $n_{\min}$ grows to infinity. Hence, it suffices to compare the first term in \eqref{eqn:MSEdecomposition} in contrast to $\textrm{MSE}(\widehat{\tau}_{w^*})$.    To elaborate the first term, we examine three scenarios previously introduced in this section, 
differentiated by the magnitude of $b_h$.

\textbf{Small $b_h$}. In this scenario, we assume $|b_h|\ll n_{\min}^{-1/2}R_{\max}/\epsilon$ and thus,  the first term is asymptotically equivalent to
\begin{eqnarray}\label{eqn:sce}
    \textrm{SEE}(\widehat{b}_h)=\Mean [(1-w^*)^2-(1-\widehat{w})^2] (\widehat{b}_h-b_h)^2. 
\end{eqnarray}
We refer to this term as the spurious estimation error (SEE) of $\widehat{b}_h$, since it occurs due to the spurious correlation between $\widehat{w}$ and $\widehat{b}_h$. Theoretically, it is of the same order of magnitude as $\textrm{MSE}(\widehat{\tau}_{w^*})$. However, our empirical investigation reveals that it is considerably smaller than the oracle MSE, as illustrated in Figure \ref{fig:SSE_difference}. 

\begin{figure}
	\centering
	\includegraphics[height=3cm, width=8.5cm]{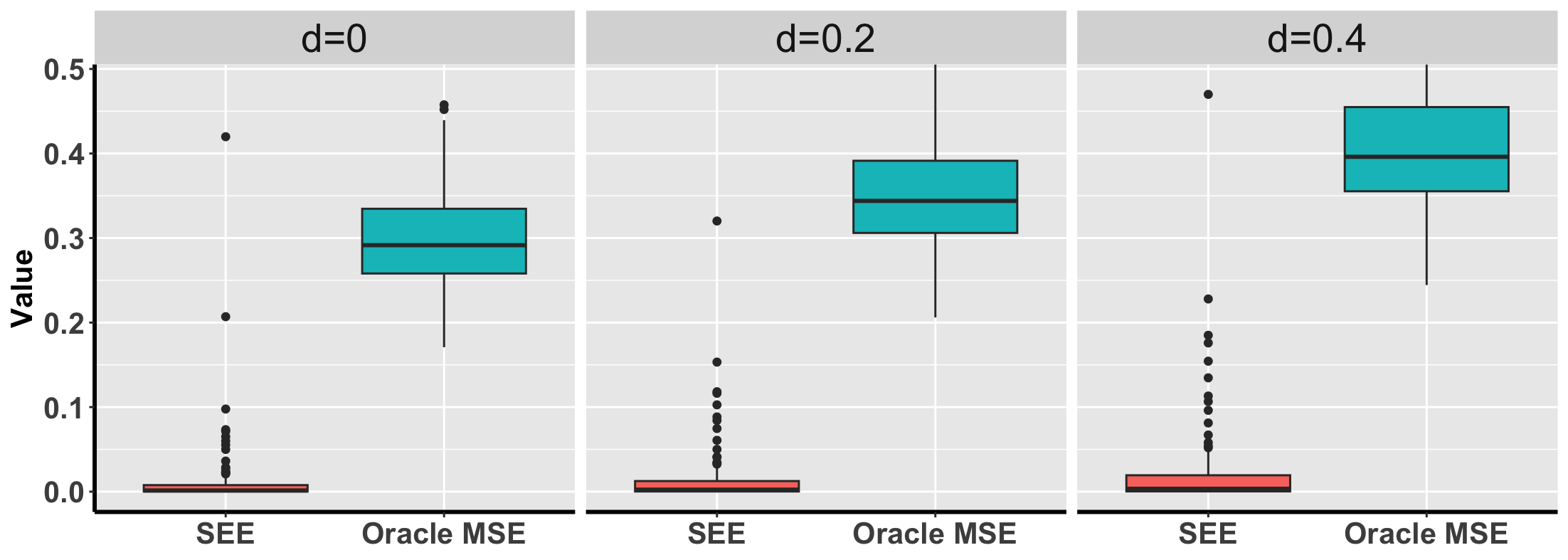}
	\vspace{-0.3 in}
	\caption{\small  Boxplots of the SEE and the oracle MSE under the setting of Example \ref{ex:single_stage} when the bias $b_h=0$, and $d$ indicates the difference of the conditional variance of the reward between the experimental data and historical data.}
	\label{fig:SSE_difference}
 \vspace{-0.1in}
\end{figure}

Additionally, under the assumption that $r_h(s)=r_e(0,s)$ for all $s$ --- effectively resulting in $b_h=0$ --- the SPE estimator achieves the smallest MSE asymptotically, since it is tailored to minimize MSE under this assumption. Assuming all nuisance functions are correctly specified, and the proportionality assumption holds such that the ratio $\Var(R_e|A_e=0,S_e)/[\Var(R_h|S_e)\mu(S_e)\pi(0|S_e)]$ remains constant irrespective of $S_e$, the SPE estimator is equivalent to the oracle estimator. Consequently, this suggests that the MSE of our non-pessimistic estimator asymptotically equals the efficiency bound augmented by a small spurious estimation error. We summarize these discussions below. 
\begin{corollary}[MSE with a small $b_h$]
\label{cor:small_bh_nonpessimistic}
    In the small $b_h$ scenario, if $\textrm{MSE}(\widehat{\tau}_{w^*})$ is proportional to $R_{\max}^2/(\epsilon^2 n_{\min})$, then 
    \begin{eqnarray*}
        \Big|\frac{\textrm{MSE}(\widehat{\tau}_{\widehat{w}})-\textrm{MSE}(\widehat{\tau}_{w^*})}{\textrm{MSE}(\widehat{\tau}_{w^*})}-\frac{\textrm{SEE}(\widehat{b}_h)}{\textrm{MSE}(\widehat{\tau}_{w^*})}\Big|\to 0,
    \end{eqnarray*}
    as $n_{\min}\to \infty$. Additionally, when $r_h(\bullet)=r_e(0,\bullet)$, the proportionality assumption holds, and all nuisance functions are correctly specified, $\textrm{MSE}(\widehat{\tau}_{\widehat{w}})$ is asymptotically equivalent to the sum of the efficiency bound plus $\textrm{SEE}(\widehat{b}_h)$. 
\end{corollary}

\textbf{Large $b_h$}. In this scenario, we require $|b_h|\gg n_{\min}^{-1/2}\sqrt{\log n_{\min}} R_{\max}/\epsilon$. Notice that the lower bound is aligned with the high confidence bound for the estimation error of $\widehat{b}_h$. Consequently, the reward shift is sufficiently large to be ``detectable" from the data. 

Under this condition, the $b_h^2$ term becomes the dominant factor in the MSE \eqref{eqn:tauw}, leading to the optimal weight $w^*$ approaching $1$. Consequently, the EDO estimator is asymptotically equivalent to the oracle estimator whereas the SPE estimator is sub-optimal since it assumes a zero $b_h$.

In the large $b_h$ scenario, the weight selected by the non-pessimistic estimator tends towards one, so that the excess MSE is of a small order. Hence, the MSE of the non-pessimistic estimator is asymptotically the same as that of the oracle estimator, achieving the oracle property. 
\begin{corollary}[Oracle property with a large $b_h$]
\label{cor:moderate_bh_nonPessimistic}
    In the large $b_h$ scenario, both $\textrm{MSE}(\widehat{\tau}_{\widehat{w}})/\textrm{MSE}(\widehat{\tau}_{w^*})$ and $\textrm{MSE}(\widehat{\tau}_{\widehat{w}})/\textrm{MSE}(\widehat{\tau}_{e})$ approach $1$ as $n_{\min}\to \infty$. 
\end{corollary}

\textbf{Moderate $b_h$}. In this scenario, the magnitude of $|b_h|$ falls 
between $n_{\min}^{-1/2}R_{\max}/\epsilon$ and $n_{\min}^{-1/2}\sqrt{\log n_{\min}} R_{\max}/\epsilon$. This scenario is the most challenging, as it is not clear whether the SPE estimator or the EDO estimator will deliver superior performance.

To illustrate the issues of the non-pessimistic estimator, let us examine a scenario where $|b_h|$ significantly exceeds $n_{\min}^{-1/2}R_{\max}/\epsilon$, causing the optimal weight $w^*\to 1$. In this context, even though $|b_h|$ is considerably large, it remains within the high-confidence interval of the estimation error for $\widehat{b}_h-b_h$, which might not make $|b_h|$ adequately distinguishable from the data. As $w^*\to 1$, the dominant factor in the first part of  \eqref{eqn:MSEdecomposition} becomes $-\Mean (1-\widehat{w})^2(\widehat{b}_h^2-b_h^2)$. Nevertheless, there is no guarantee that $\widehat{w}$ will converge to 1 with high confidence. This uncertainty introduces a significant excess MSE for the non-pessimistic estimator. See the numerical results in Section \ref{sec:numerical}.

\subsection{Robustness of the Pessimistic Estimator}\label{sec:pessitheory}
The pessimistic estimator effectively mitigates the aforementioned limitation of the non-pessimistic estimator by incorporating the estimation error of $\widehat{b}_h$ into weight selection. To elaborate, we first provide a non-asymptotic upper bound for its MSE. 
\begin{theorem}[MSE of the pessimistic estimator]\label{thm:pessi}
    Under Assumptions \ref{con:coverage} -- \ref{con:double} and \eqref{eqn:quantifier}, $\textrm{MSE}(\widehat{\tau}_{\widehat{w}_U})-\textrm{MSE}(\widehat{\tau}_{w^*})$ can be upper bounded by
\begin{eqnarray}\label{eqn:MSEpessi}
\begin{split}
    (1-w^*)^2 \Mean [(|\widehat{b}_h|+U)^2-b_h^2]+O\Big(\frac{R_{\max}^2}{\epsilon^2 n_{\min}^{3/2}}\Big)\\
    +O(\alpha [b_h^2+R_{\max}^2/\epsilon^2 n_{\min}]). 
\end{split}
\end{eqnarray}
\end{theorem}
According to Theorem \ref{thm:pessi}, the excess MSE of the pessimistic estimator can be decomposed into three parts: \vspace{-0.5em}
\begin{enumerate}
    \item \textbf{Estimation error of} $\widehat{b}_h$: the first term quantifies the estimation error for $\widehat{b}_h$. Unlike the non-pessimistic estimator where this term depends on the estimated weight, here it relies only on $w^*$. This distinction enhances the robustness of the estimator, particularly when $b_h$ is moderately large. \vspace{-0.5em}
    \item \textbf{Estimation errors of the variance/covariance terms}: similar to the non-pessimistic estimator, the second term quantifies the estimation error for the variance/covariance terms.\vspace{-0.5em}
    \item \textbf{Type-I error}: the last term is directly proportional to the type-I error $\alpha$ which upper bounds the
    probability that the $|\widehat{b}_h-b_h|$ exceeds $U$. Notice that this term can be made sufficiently small by employing concentration inequalities without substantially increasing  the estimation error associated with $\widehat{b}_h$.\vspace{-0.5em}
\end{enumerate}
To further illustrate the advantage of the pessimistic estimator, we consider the moderate $b_h$ scenario. 
When $|b_h|\gg n_{\min}^{-1/2}R_{\max}/\epsilon$, $\widehat{w}$ might not necessarily converge to $1$ with high confidence. Hence, the non-pessimistic estimator can suffer from a large loss, due to the involvement of $\widehat{w}$ in \eqref{eqn:MSEdecomposition}. On the contrary, \eqref{eqn:MSEpessi} depends solely on $w^*$, which results in a smaller excess loss. Indeed, the following corollary shows that the pessimistic estimator achieves the oracle property even when $b_h$ is moderately large. 
\begin{corollary}[Oracle property of the pessimistic estimator]\label{coro:pessimistic}
    Suppose that $U$ is proportional to the order $n_{\min}^{-1/2}\sqrt{\log n_{\min}} R_{\max}/\epsilon$ such that $\alpha=o(1/n_{\min})$, if further
    $b_h\gg n_{\min}^{-1/2} (\log n_{\min})^{1/6} R_{\max}/\epsilon$, then the pessimistic estimator achieves the oracle property. 
\end{corollary}
The condition $b_h\gg n_{\min}^{-1/2} (\log n_{\min})^{1/6} R_{\max}/\epsilon$ applies to both the moderate and large $b_h$ scenarios.
Hence, even in cases of moderate $b_h$, as long as it is much larger than $n_{\min}^{-1/2} (\log n_{\min})^{1/6} R_{\max}/\epsilon$, the oracle property is satisfied. 
This formally establishes the robustness of the pessimistic estimator in comparison to the non-pessimistic estimator.

\section{Experiments}
\label{sec:numerical}
In this section, we investigate the finite sample performance of the proposed estimators.
Comparison is made among the following ATE estimators:
\begin{itemize}
    \item \textbf{NonPessi}: the proposed non-pessimistic estimator.
    \item 
    \textbf{Pessi}: the proposed pessimistic estimator.
    \item 
    \textbf{EDO}: the doubly robust estimator $\widehat \tau_e$ constructed based on the experimental data only (see \eqref{eqn:taue}).
    \item 
    \textbf{Lasso}: a weighted estimator $\widehat \tau_{Lasso}=w\widehat \tau_e + (1-w) \widehat \tau_h$ that linearly combines the ATE estimator $\widehat \tau_e$ based on experimental data and $\widehat \tau_h$ based on historical data, where the weight $w$ is chosen to minimize the estimated variance of the final ATE estimator with the Lasso penalty \cite{cheng2021adaptive},
\item 
\textbf{SPE}: the semi-parametrically efficient estimator proposed by \citet{li2023improving} developed under the assumption of no reward shift between the experimental and historical data, i.e., $r_e(0,s)=r_h(s)$ for any $s$. 
\end{itemize}
Notice that it remains unclear how to extend SPE in sequential decision making. Consequently, our implementation of SPE is confined to non-dynamic settings only. 
We compare the MSEs of the ATE estimators based on 100 simulation replications. Details about the data generating process can be found in Appendix \ref{sec:additional experiments}.

\begin{example}[\bf Non-dynamic simulation]
\label{ex:single_stage}
We consider a non-dynamic setting where the sample size of the experimental data is $|\mathcal{D}_e|=48$, and the sample size of the historical data is set to be $|\mathcal{D}_h|=m|\mathcal{D}_e|$ with {$m \in \{1,2,3\}$.} A deterministic switchback design is adopted to generate $\mathcal{D}_e$.
We 
vary the mean reward shift $b_h$ 
within the range from 0 to 1.5, incrementing by 0.1 at each step. We also vary the conditional variance of the reward and use $d$ to characterize this difference (see Appendix \ref{sec:additional experiments} for its detailed definition). 

Figure \ref{fig:single_withSPE} visualizes the empirical means of the MSEs for different methods.
According to our theory, the effectiveness of different estimators is determined by the magnitude of the reward shift. To validate our theory, we further classify $b_h$ into different regimes as follows:\vspace{-1em}
\begin{itemize}
    \item \textbf{Small $b_h$ regime}: $| b_h| \leq  c_1\sqrt{\Var(\widehat b_h)} $;\vspace{-0.5em}
    \item \textbf{Moderately large $b_h$ regime}: $c_1<\frac{| b_h|}{\sqrt{\Var(\widehat b_h)}} \leq  c_2$;\vspace{-0.5em}
    \item \textbf{Large $b_h$ regime}: $| b_h|> c_2\sqrt{\Var(\widehat b_h)}$.\vspace{-1em}
\end{itemize}
According to our theoretical analysis, we set $c_1=1$ and $c_2=\sqrt{\log(n_{\min})}$. This ensures that scenarios where variance dominates the bias are categorized within the small reward shift region. Conversely, when the bias exceeds the established high confidence bound, it is classified under the large reward shift regime. 

We depict the boundaries between different regimes in Figure \ref{fig:single_withSPE}.
It can be seen that in the small $b_h$ regime, the SPE estimator is the top performer. However, the MSE of our proposed non-pessimistic estimator is close to that of the SPE estimator. As $b_h$ grows to moderate levels, our pessimistic estimator achieves smaller or comparable MSEs compared to other alternatives. Finally, in the large $b_h$ regime, our pessimistic estimator achieves comparable performance to the EDO estimator, both outperforming other estimators in terms of MSE. These findings 
establish a concrete link between our theories and empirical observations. Particularly, 
they numerically verify our theoretically identified optimal method within each respective regime. 
Additionally, the bottom panel of Figure \ref{fig:single_withSPE} specifically reports the mean MSEs for methods excluding SPE, offering an in-depth comparison of the other estimators' performance. Here, Lasso is implemented with a carefully selected tuning parameter, which has been determined to yield reasonably good performance. However, as illustrated by additional numerical results in Figure \ref{fig:real_data_based_Lasso_all} in the Appendix, this estimator is sensitive to the choice of the tuning parameter.

\begin{figure}[t]
	\centering
	\includegraphics[height=5.25cm, width=8.5cm]{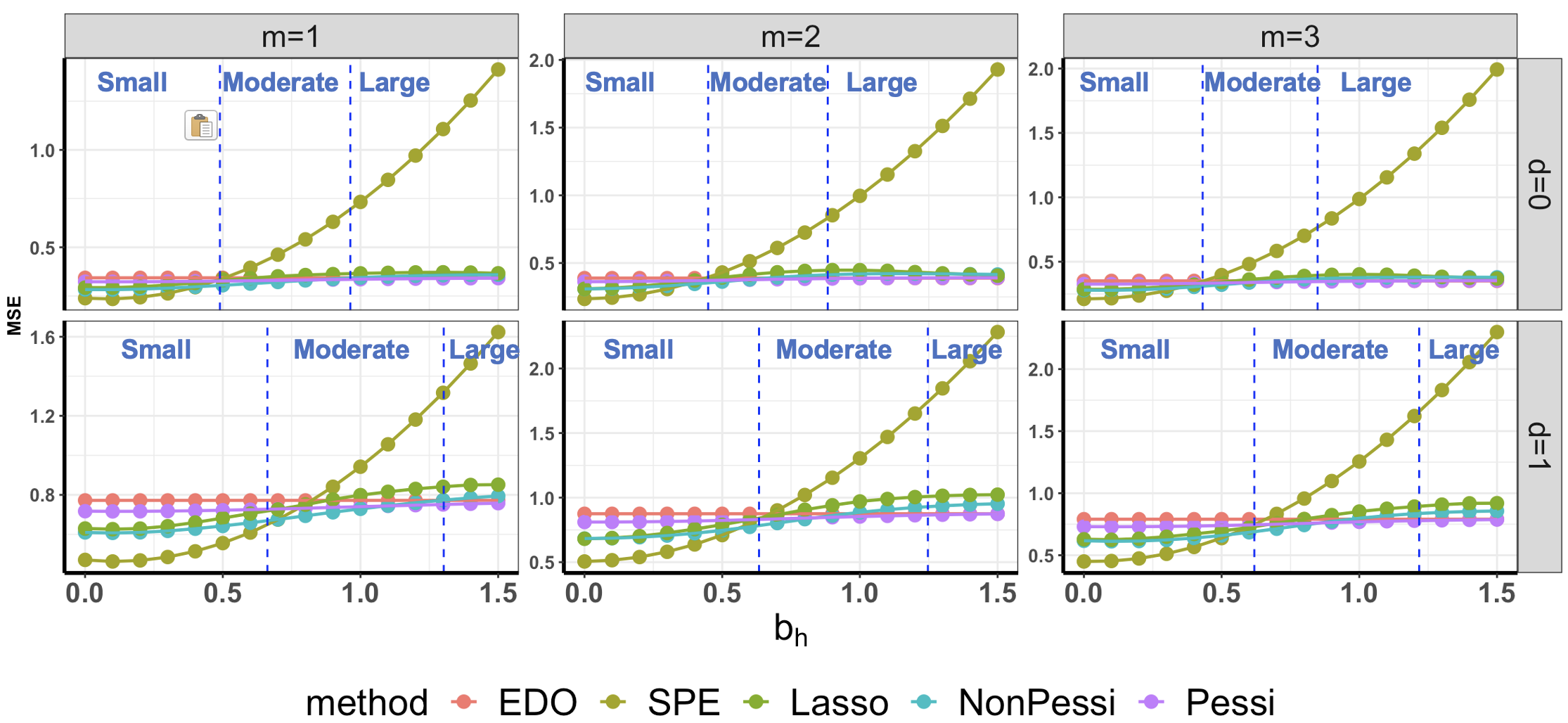}\\
 	\includegraphics[height=5.25cm, width=8.5cm]{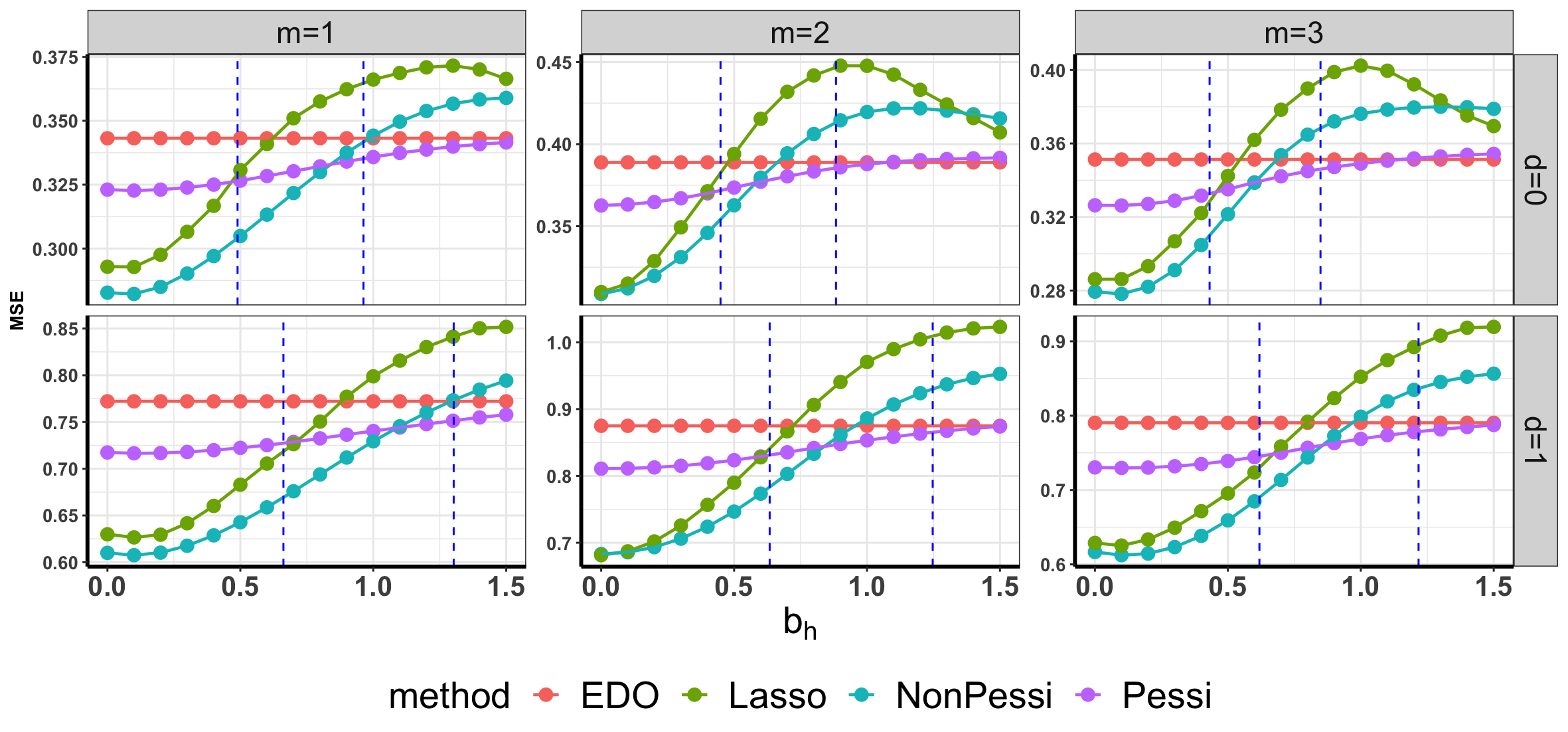}
	\vspace{-0.3 in}
	\caption{\small  Empirical means of MSEs for different methods under the switchback design in Example \ref{ex:single_stage}. The top panel displays all the methods, whereas the bottom panel focuses on the area excluding the SPE method. }
	\label{fig:single_withSPE}
  \vspace{-0.2 in}
\end{figure}

\end{example}

\begin{example}[\bf Ridesharing-data based sequential simulation]
\label{ex:real_data_based_agnostic}
In this example, we build a simulation environment based on a real dataset collected from a ridesharing company. The experimental data lasts for $|\mathcal{D}_e|=30$ days and is generated from a switchback design. 
We divide each day into $T=24$ time intervals. 
The state variable consists of the number of order requests and the driver's total online time within each one-hour time interval. The reward is defined as the total income earned by the drivers within each time interval.
To generate the historical data, we assume it contains another $|\mathcal{D}_h|=m|\mathcal{D}_e|$ days with $m \in \{1,2,3\}$, and set $b_h$ as a linearly increasing sequence ranging from 0 to 0.3, consisting of eight values, and choose the conditional variance difference parameter $d$ from the set $\{0,0.5\}$.  Refer to Appendix \ref{sec:additional experiments} for the detailed data generating process. 

\begin{figure}[t]
	\centering
	\includegraphics[height=5.25cm, width=8.5cm]{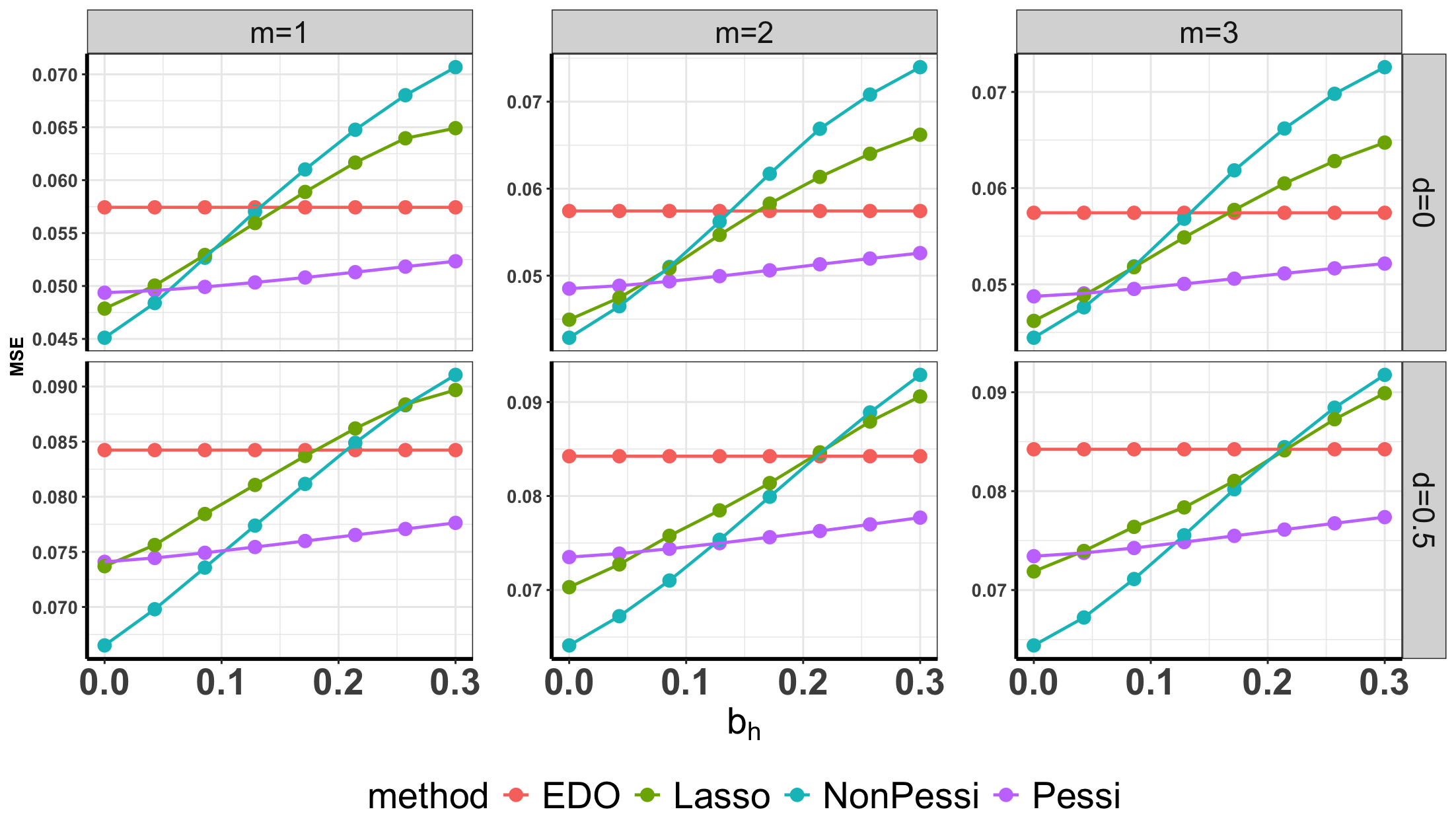}
 	\includegraphics[height=5.25cm, width=8.5cm]{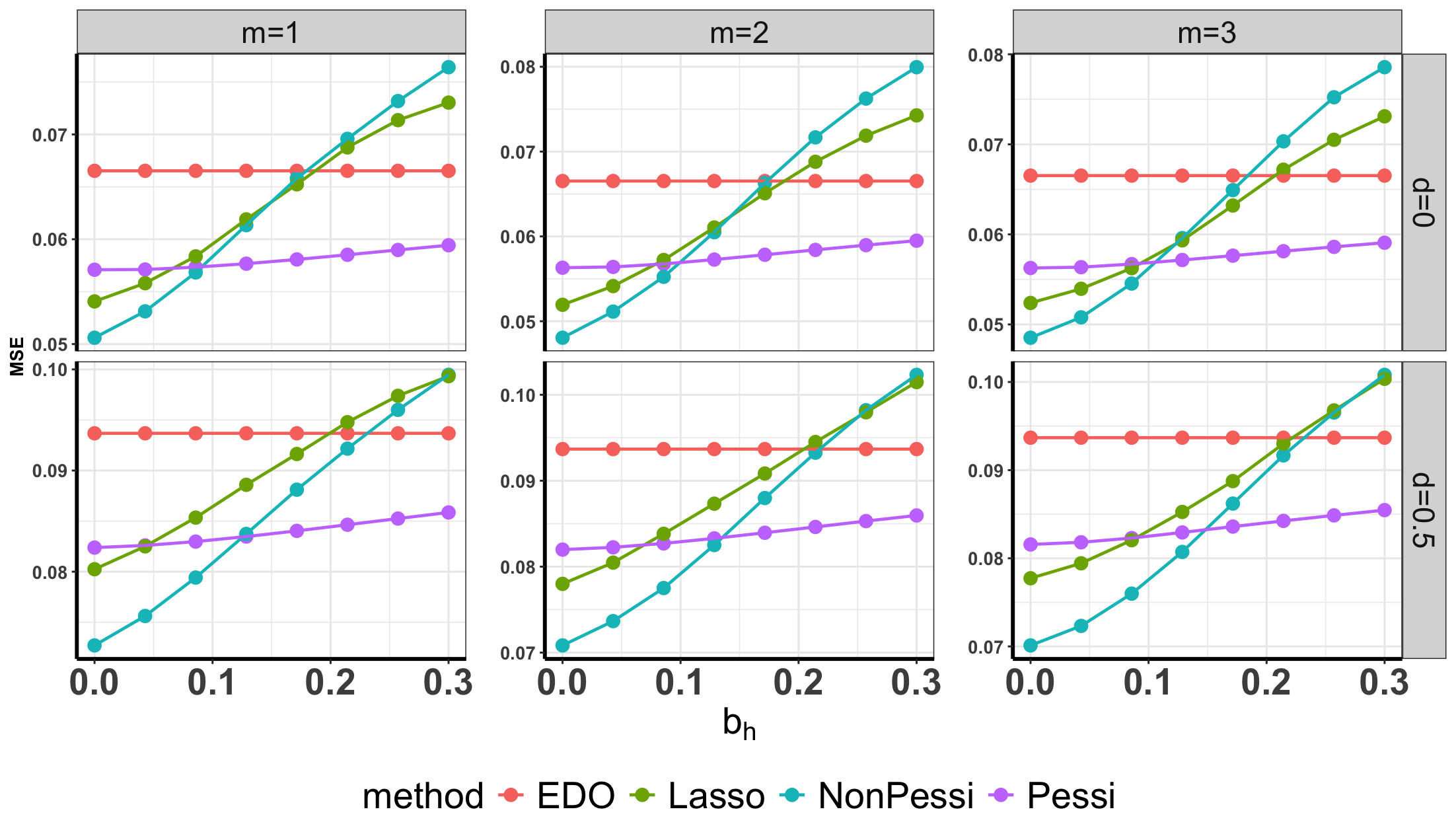}
	\vspace{-0.3 in}
	\caption{\small  Empirical Means of MSEs for different methods in Example \ref{ex:real_data_based_agnostic}. The treatment effect ratios are equal to $5\%$ (Top), $10\%$ (Bottom), respectively.}
 \vspace{-0.2 in}
\label{fig:real_data_based_agnostic_all}
\end{figure}

Figure \ref{fig:real_data_based_agnostic_all} reports the empirical means of the MSEs of different estimators. 
The Lasso method is again implemented with a reasonably good tuning parameter. 
It can be seen that the proposed ``NonPessi" estimator outperforms Lasso in most cases. When $b_h$ is large, the two methods perform comparably. 
In contrast, our proposed ``Pessi'' showcases robustness in dealing with the distributional shift. 
When $b_h$ increases, the MSEs of Lasso and ``NonPessi'' increase significantly, while ``Pessi'' has only a slight increase in the MSE. It is also consistently better than the EDO estimator, demonstrating the usefulness of data integration.  
\end{example}

\begin{example}[{\bf Ridesharing-data based non-dynamic simulation}]
\label{ex:real_data_based_non_dynamic}
    We also conduct a real-data based non-dynamic simulation study to compare different estimators. The findings are very similar. To save space, we relegate the detailed results to Appendix \ref{sec:additional experiments}. 
\end{example}

\begin{example}[{\bf Clinical-data based non-dynamic simulation}]
\label{ex:clinical_data_based_non_dynamic}
The data for this experiment is sourced from the AIDS Clinical Trials Group Protocol 175, involving 2139 HIV-infected individuals. Participants were randomly assigned to one of four treatment groups: zidovudine (ZDV) monotherapy, ZDV+didanosine (ddI), ZDV+zalcitabine, or ddI monotherapy \citep{hammer1996trial}. Following the analyses in \cite{lu2013variable} and \cite{shi2019testing}, we use the CD4 count (cells/mm3) at 20 ± 5 weeks post-baseline as the continuous response variable, and consider three contextual variables: age (in years), homosexual activity (0=no, 1=yes), and hemophilia (0=no, 1=yes).

Based on this clinical dataset, we create a simulation environment to test our methodology. The findings from this clinical-data-based non-dynamic experiment align with the patterns observed in the above examples. Details about the experimental settings and figures summarizing the findings can be found in Appendix \ref{sec:additional experiments}.
\end{example}

Furthermore, we conduct an additional experiment in Example \ref{ex:inference} of Appendix \ref{sec:additional experiments}  to evaluate the coverage probabilities of the confidence intervals (CIs). While maintaining nominal coverage, the pessimistic estimator yields narrower confidence intervals compared to the EDO estimator, indicating an improvement in efficiency by incorporating historical data.
 We also develop a hybrid procedure that chooses different methods according to the magnitude of the bias in Appendix \ref{sec:hybrid_procedure}.

\section*{Acknowledgement}
We thank the anonymous referees and the meta reviewer for their constructive comments, which have led to a significant improvement of the earlier version of this article.
Li’s research is partially supported by the National Science Foundation of China 12101388, CCF- DiDi GAIA Collaborative Research Funds for Young Scholars and Program for Innovative Research Team of Shanghai University of Finance and Economics. Shi’s research is partially supported by an EPSRC grant EP/W014971/1.

\section*{Impact Statement}
This paper introduces innovative methods for policy evaluation, particularly focusing on the integration of multiple data sources to enhance decision-making processes. 
There are many potential societal consequences of our work, none which we feel must be specifically highlighted here.

\bibliography{DataCombination_refs}
\bibliographystyle{icml2024}

\newpage
\appendix
\onecolumn
\renewcommand{\thefigure}{A\arabic{figure}}
\setcounter{figure}{0}

\section{Additional Experiment Results}
\label{sec:additional experiments}
In this section, we present details of the data-generating process for Section \ref{sec:numerical} and additional experiment results.

\textit{Example \ref{ex:single_stage}} {\bf (Continued).}
We consider the reward function as follows,
\begin{eqnarray*}
R_{e} = 10+  b_h  +A_{e} +  S_{e} + (2 + d) \varepsilon_{e}, \quad
R_h = 10 + S_h + \varepsilon_{h},
\end{eqnarray*}
where $S_{e}, S_h$'s and $\varepsilon_{e}, \varepsilon_{h}$'s are from standard normal distribution $N(0,1)$.
The sample size $|\mathcal{D}_e|=48$ with a horizon of $T=1$. 
The sample size of historical data is set to be $|\mathcal{D}_h|=m|\mathcal{D}_e|$ with $m \in \{ 1,2,3 \}$.
We consider the switchback design, which alternates the treatment and the control along the time.
We set $b_h$ to range over the set $\{0, 0.1, 0.2, \dots 1.5\}$. A larger value of $b_h$ indicates a greater difference in the average cumulative reward of the control between the historical data and the experimental data. Meanwhile, we use $d \in \{0, 1\}$ to characterize the difference of the conditional variance of the reward between the historical data and the experimental data. 

{\textit{Example \ref{ex:real_data_based_agnostic}}} {\bf(Continued).}
We perform the analysis based on the real dataset obtained from a prominent ridesharing company. The dataset covers the period from May 17th, 2019, to June 25th, 2019, with one-hour time units, resulting in a total of $T=24$ hours per day, collected over the span of $N=40$ days. To protect privacy, specific details about the company and cities are omitted, and all data, including states and rewards, are scaled to ensure privacy. 
In particular, the state variable consists of the number of order requests and the driver's total online time within each one-hour time interval. The reward is defined as the total income earned by the drivers within each time interval. Noting daily temporal trends in these variables, as shown in Figure \ref{fig:cityAB_t_pattern}, a time-varying Markov Decision Process (MDP) model is employed to understand the dataset dynamics. The dataset is based on the A/A experiments, in which a single order dispatch policy was consistently applied over time ($A_t=0$ for all $t$).
Similar to the data generating process in \citet{luo2022policy},
we create a simulation environment using the wild bootstrap method \citep{wu1986jackknife} based on this A/A dataset. In general, we assume the following time-varying linear models:
\begin{equation}\label{linear_mdp}
\left\{ 
\begin{split}
R_{t} &=\alpha_t + S_t^\top \beta_t + \gamma_t A_t +e_t, \\
S_{t+1}&=\phi_t + \Phi_t S_t + \Gamma_t A_t +E_t,
\end{split}
\right. 
\end{equation}
where $\alpha_t$ and $\gamma_t$ are real-valued scalars, $\beta_t,\phi_t,$ and $\Gamma_t$ are vectors in the space $\mathbb{R}^d$,  $\Phi_t \in \mathbb{R}^{d\times d}$, $e_t$ is the time-dependent random noise, and $E_t$ is the time-independent random error vector. Specifically, we first fit the data based on the linear models in \eqref{linear_mdp} by setting $\gamma_t=\Gamma_t=0$ and derive the estimates 
$\{\widehat{\alpha}_t\}_t$, $\{\widehat{\beta}_t\}_t$, $\{\widehat{\phi}_t\}_t$ and $\{\widehat{\Phi}_t\}_t$. 
We then calculate the residuals in the reward and state regression models based on these estimators as follows:  
\begin{equation}\label{eqn:residuals}
	 \widehat{e}_{i,t}=R_{i,t}-\widehat{\alpha}_t-S_{i,t}^\top \widehat{\beta}_t, \quad \widehat{E}_{i,t}=S_{i,t+1}-\widehat{\phi}_{t}- \widehat{\Phi}_t S_{i,t}.
\end{equation}

To simulate data reflecting varied treatment effects, we introduce a treatment effect ratio $\lambda$, which is defined as the proportional impact on the average return of the baseline policy. Define treatment effect parameters {$\widehat{\gamma}_t=\delta_1\times (\sum_i R_{i,t}/(100\times N))$ and $\widehat{\Gamma}_t=\delta_2\times (\sum_i S_{i,t}/(100\times N))$}.  We examine $\lambda$ at $5\%$ and $10\%$ levels. Correspondingly, $\delta_1$ and $\delta_2$ are adjusted to ensure that both the direct and carryover effects increment by $\lambda/2$, cumulatively elevating the Average Treatment Effect (ATE) by $\lambda$.


To structure the experimental ($|\mathcal{D}_e|=30$ days) and historical ($|\mathcal{D}_h|=m |\mathcal{D}_e|$ days) datasets with $m \in \{1,2,3\}$, we employ an alternating time interval design with a 3-hour span for the experimental dataset, and a global control for the historical dataset (i.e., $A_{i,t}=0$ for all $i, t$). We introduce i.i.d. standard Gaussian noise $\{\xi_i\}_{i=1}^{|D_e|}$ or $\{\xi_i\}_{i=1}^{|D_h|}$ for each dataset. For each day $i$, a random integer from set $I$ (where $I \in  \{1, 2, \dots, N \}$) is selected to determine the initial state $S_{I,1}$. For the historical dataset, the $b$-th bootstrap sample of rewards and states is generated following specific equations.
\begin{equation*}
\left\{ 
\begin{split}
\widehat{R}_{i,t}^b &=\widehat{\alpha}_t + (\widehat{S}_{i,t}^b)^\top \widehat{\beta}_t + \widehat{\gamma}_t A_{i,t} +\xi_i^b \hat{e}_{i,t}, \\
\widehat{S}_{i,t+1}^b &=\widehat{\phi}_t + \widehat{\Phi}_t \widehat{S}_{i,t}^b + \widehat{\Gamma}_t A_{i,t} +\xi_i^b \widehat{E}_{i, t},
\end{split}
\right. 
\end{equation*}
with the estimated $\{\widehat{\alpha}_t\}_t$, $\{\widehat{\beta}_t\}_t$, $\{\widehat{\phi}_t\}_t$, $\{\widehat{\Phi}_t\}_t$, the specified $\{\widehat{\gamma}_t\}_t$ and $\{\widehat{\Gamma}_t\}_t$, and the error residuals given by $\{\xi_i^b \widehat{e}_{i,t}:1\le t\le T\}$ and $\{\xi_i^b \widehat{E}_{i,t}:1\le t\le T\}$, respectively. In the experimental dataset, we exclusively generate the $b$-th Bootstrap sample of reward with shifted mean parameter $b_h$ and standard deviation parameter $d$, according to the equation, $\widehat{R}_{i,t}^b =\widehat{\alpha}_t + (\widehat{S}_{i,t}^b)^\top \widehat{\beta}_t + \widehat{\gamma}_t A_{i,t}+b_h +(1+d) \xi
_i^b \widehat{e}_{i,t}$, and we continue to generate states based on the settings of historical data. A summary of the bootstrap-assisted procedure is provided in Algorithm \ref{algo:res_bootstrap}.

\begin{algorithm}[h]
	\caption{Bootstrap-assisted procedure.}
 \label{algo:res_bootstrap}
 \begin{algorithmic}[1]
 \Require
Real data $\left\lbrace (S_{it}, R_{it}): 1 \leq i \leq N; 1 \leq t \leq T \right\rbrace$, the adjustment parameters for the ratios $(\delta_1, \delta_2)$, the assignment of actions, the bootstrapped sample size ($n=|D_e| \text{ or } n=|D_h|$, where $|D_h|=m |D_e|$), shifted mean parameter $b_h$ and standard deviation parameter $d$, random seed, the number of replications $B=200$.

\State 
\textbf{Initialization:} Calculating the least square estimates $\left\lbrace
	\hat{\alpha} \right\rbrace_t $,  $\left\lbrace
	\hat{\beta}_t \right\rbrace_t $,  $\left\lbrace
	\hat{\phi}_t \right\rbrace_t $,  $\left\lbrace
	\hat{\Phi}_t \right\rbrace_t $  in the model (\ref{linear_mdp}), treatment effect parameters  $\left\lbrace \hat{\gamma}_t \right\rbrace_t $ and  $\left\lbrace
	\hat{\Gamma}_t \right\rbrace_t $ and the residuals of reward model and state regression model by the  \eqref{eqn:residuals}\;
	
\For{$1\leq b\leq B$}
\State 
		Sampling the number of days $n$ from $\left\lbrace 1, \cdots ,N \right\rbrace $ with replacement, and generating the i.i.d. normal random variables $  \xi_{i}^b  \sim N(0,1)$;
\State
		Generating the pseudo rewards$\left\lbrace \hat{R}_{i,t}^b \right\rbrace_{i,t} $ and states $\left\lbrace \hat{S}_{i,t}^b \right\rbrace_{i,t} $ according to the following equations, 
			\begin{equation*}
		\hat{R}_{i,t}^b=[1, (\hat{S}_{i,t}^{b})^\top, A_{i,t} ] \begin{pmatrix}
			\hat{\alpha}_t \\
			\hat{\beta}_t \\
			\hat{\gamma}_t
		\end{pmatrix}+D_ib_h  +(1+D_i d)\xi_{i}^b \hat{e}_{i,t}, \quad 
		\hat{S}_{i,t+1}^b=[\hat{\phi}_t, \hat{\Phi}_t, \hat{\Gamma}_t ] \begin{pmatrix}
			1\\
			\hat{S}_{i,t}^{b} \\
			A_{i,t}
		\end{pmatrix} + \xi_{i}^b \hat{E}_{i,t},
	\end{equation*}
	  where $D_i = 1$ corresponds to index in the experimental dataset, while $D_i = 0$ corresponds to index in the historical dataset\;
   \State
		Calculating the set of estimators $\left\lbrace
		\textrm{ATE}^{b}
		\right\rbrace_{b} $ by proposed methods and state-of-the-art methods.
\EndFor
 
\Ensure 
 The empirical means of MSEs of different ATE estimators.
\end{algorithmic}
\end{algorithm}

Furthermore, to explore the Lasso method's efficacy across various tuning parameters, we select a set of $\ell_1$-tuning parameters $\{0.8, 1.2, 1.6, 2, 4\}$. Additionally, we take the mean difference $b_h$ from a sequence ranging from 0 to 0.6 in increments, comprising 8 values.
Figure \ref{fig:real_data_based_Lasso_all} illustrates the performance comparison between varying $\ell_1$-tuning parameters and the EDO method. We observe that for smaller $b_h$ values, Lasso methods with lower tuning parameters outperform the EDO method. However, their performance declines as $b_h$ increases. In contrast, Lasso methods with larger tuning parameters maintain a consistent efficiency, comparable to the EDO method, regardless of $b_h$ values.

\begin{figure}[ht]
  \centering
  \includegraphics[width=0.9\textwidth,height=1.5in]{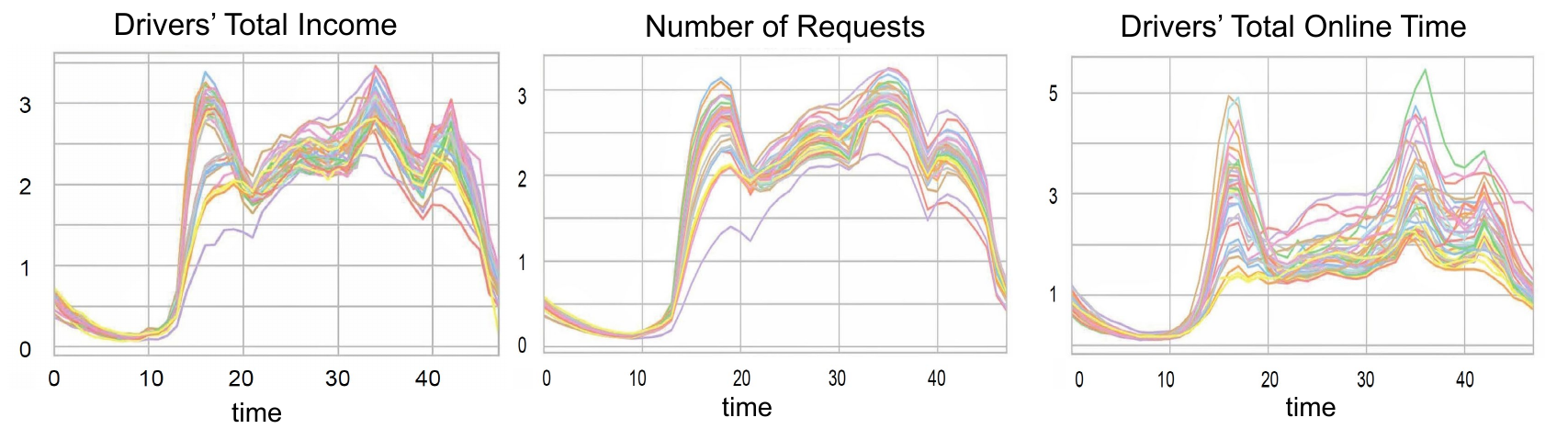}
  \caption{Visual representations of scaled states and rewards in one city across 40 days, comprising drivers' total income, the number of requests,  and drivers' total online time. Each line represents data from a specific day.}\label{fig:cityAB_t_pattern}
  \end{figure}

\begin{figure}[ht]
	\centering
      \includegraphics[height=6cm, width=8.5cm]{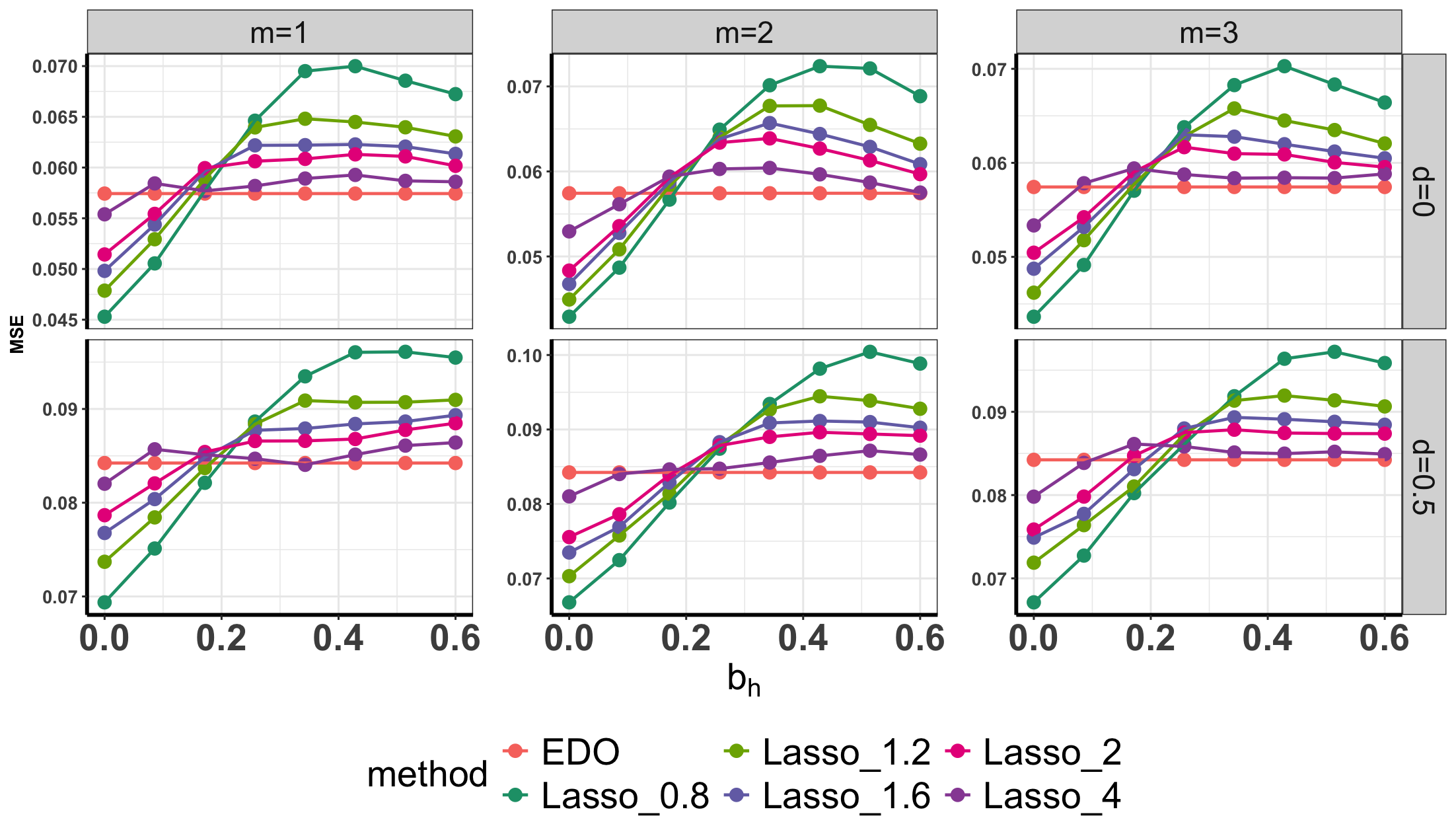}
            \includegraphics[height=6cm, width=8.5cm]{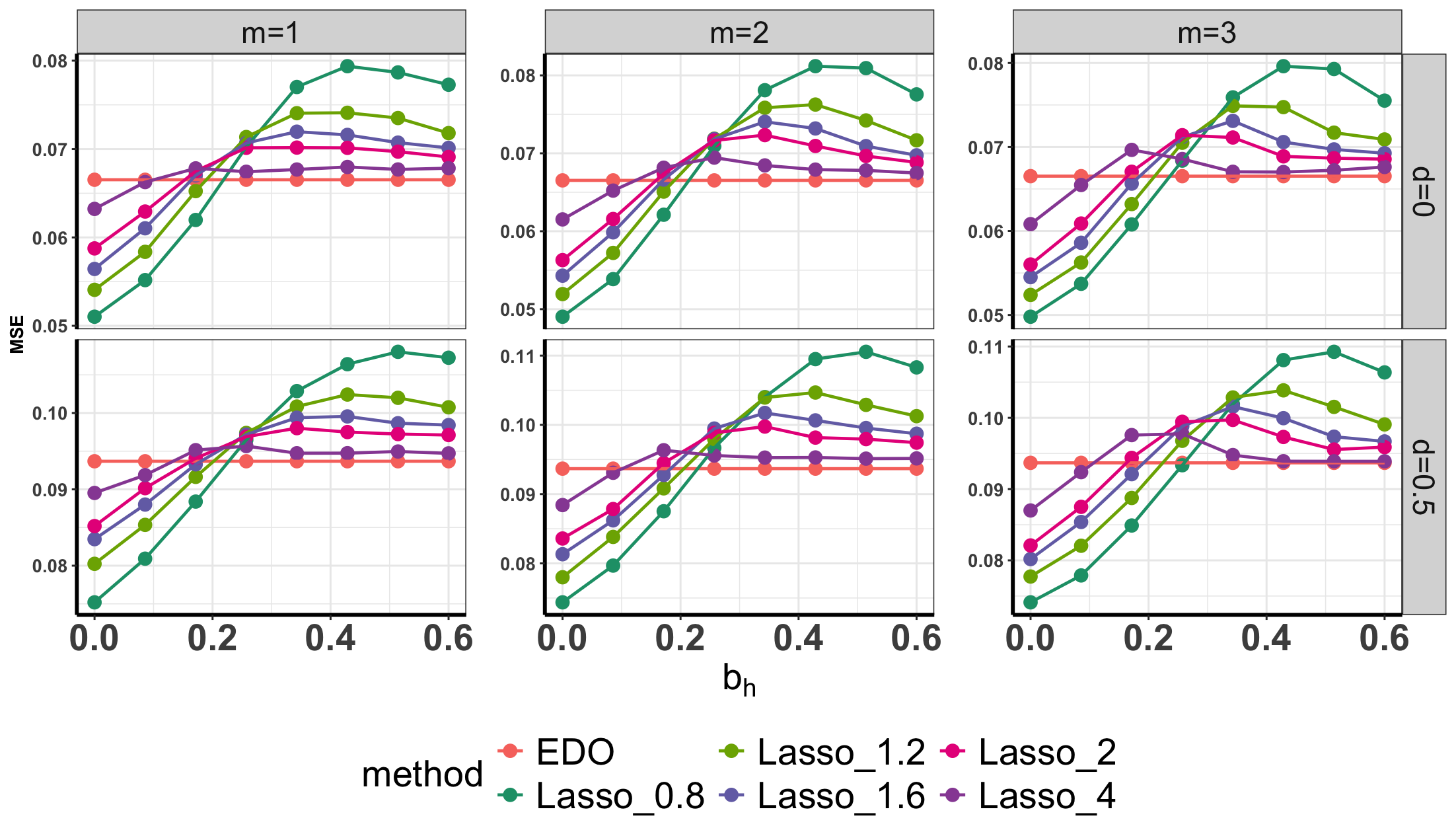}
            \vspace{-0.2 in}
	\caption{\small  Empirical Means of MSEs for Lasso with different tuning parameters under the switchback design in Example \ref{ex:real_data_based_agnostic}.  The treatment effect ratio parameters are equal to $5\%$ (Left) and  $10\%$ (Right), respectively. }
	\label{fig:real_data_based_Lasso_all}
\end{figure}

{\textit{Example \ref{ex:real_data_based_non_dynamic}}} {\bf(Continued).}
\label{ex:real_data_based_adaptive}
The data generating process is adapted from the approach outlined in Example \ref{ex:real_data_based_agnostic}.
Crucially, although the data is inherently sequential, we adapt it to a non-dynamic setting by treating each day as an independent instance with $T=1$.
This adjustment involves maintaining the same policy daily and defining
the daily average total income of drivers over all time intervals as the reward.
The state variable is represented by the number of order requests and the drivers' total online time during the initial time interval.
The sample size of the experimental data is $|\mathcal{D}_e|=30$ and that of the historical data is $|D_h|=m |D_e|$ with $m \in \{1,2,3\}$, and we take the mean difference { $ b_h $ from an arithmetic sequence ranging from $0$ to $0.6$, with a length of 8}, and the conditional variance difference $d \in \{0,0.5\}$ to explore different scenarios.

 Figure \ref{fig:single_real_data_based_adaptive_all_spe} illustrates the performance of our proposed estimators, EDO, SPE, and Lasso with a reasonably good tuning parameter. 
 The outcomes align with those observed in Example \ref{ex:single_stage}, demonstrating comparable insights. Meanwhile, Figure \ref{fig:single_real_data_based_adaptive_all} depicts the performance of these methods, excluding SPE, further validating our theoretical insights and highlighting the resilience of the "Pessi" estimator. Notably, as the treatment effect size grows, the "Pessi" approach exhibits increased stability across different values of $b_h$.

 various methods, excluding SPE, under the switchback design as detailed in Example \ref{ex:real_data_based_non_dynamic}. The figures display results for treatment effect ratio parameters set at $5\%$ (left) and $10\%$ (right) respectively.
\begin{figure}[ht]
	\centering
 	\includegraphics[height=6cm, width=8.5cm]{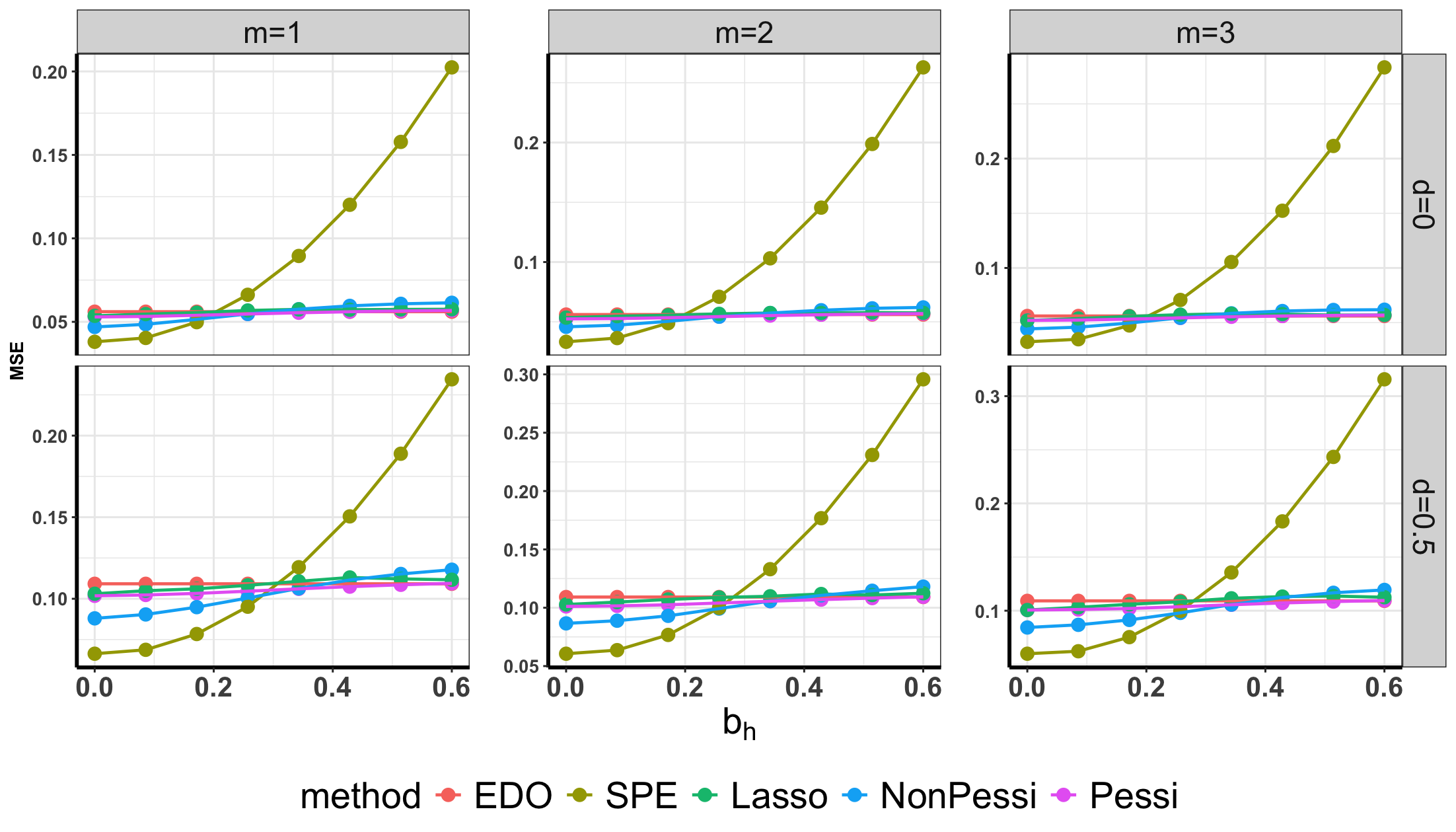}
      \includegraphics[height=6cm, width=8.5cm]{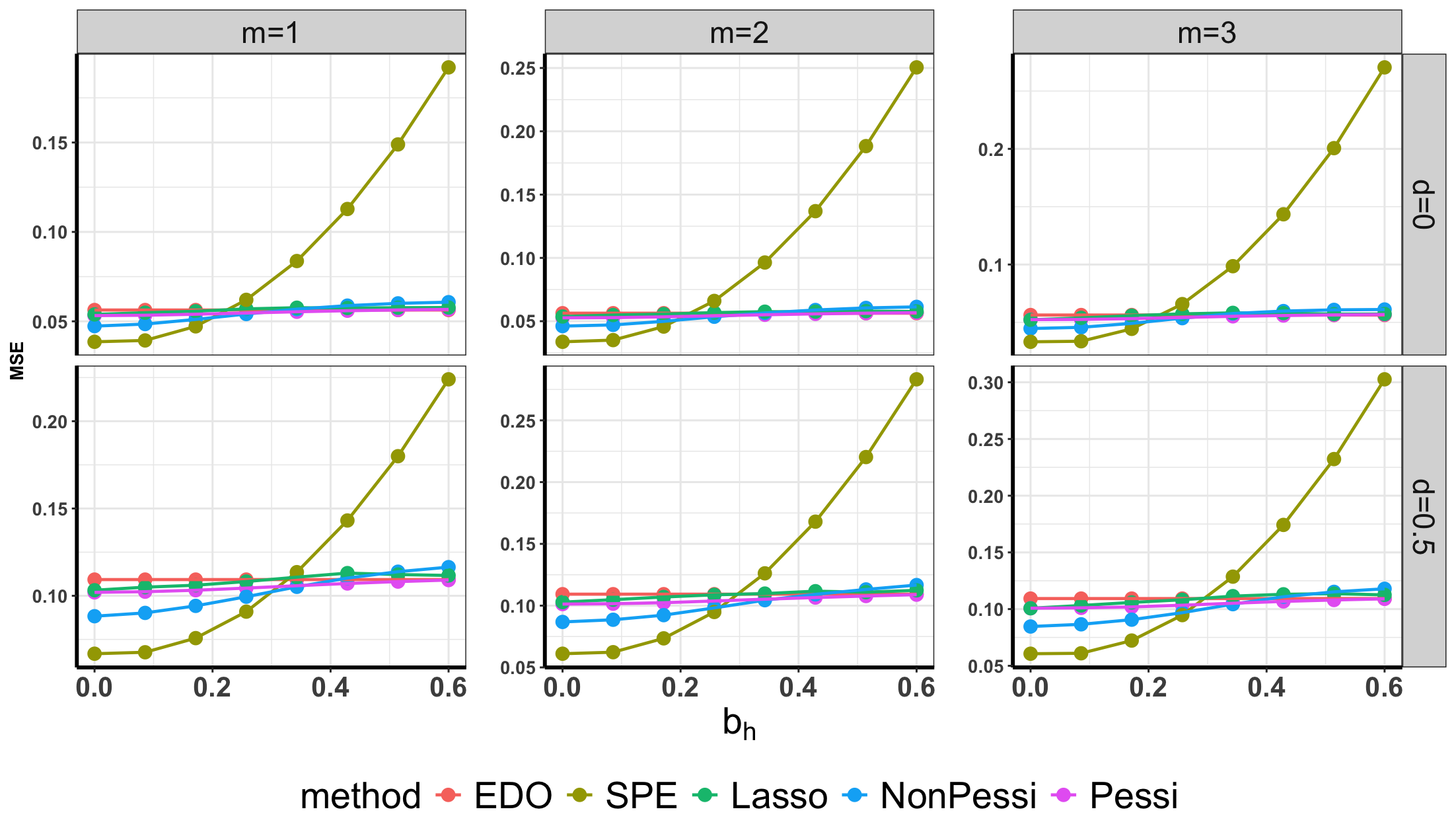}
	\caption{\small  Empirical means of MSEs for  various methods under the switchback design as detailed in Example \ref{ex:real_data_based_non_dynamic}.  The figures display results for treatment effect ratio parameters set at $5\%$ (left) and $10\%$ (right), respectively.
 }
	\label{fig:single_real_data_based_adaptive_all_spe}
\end{figure}
\begin{figure}[ht]
	\centering
      \includegraphics[height=6cm, width=8.5cm]{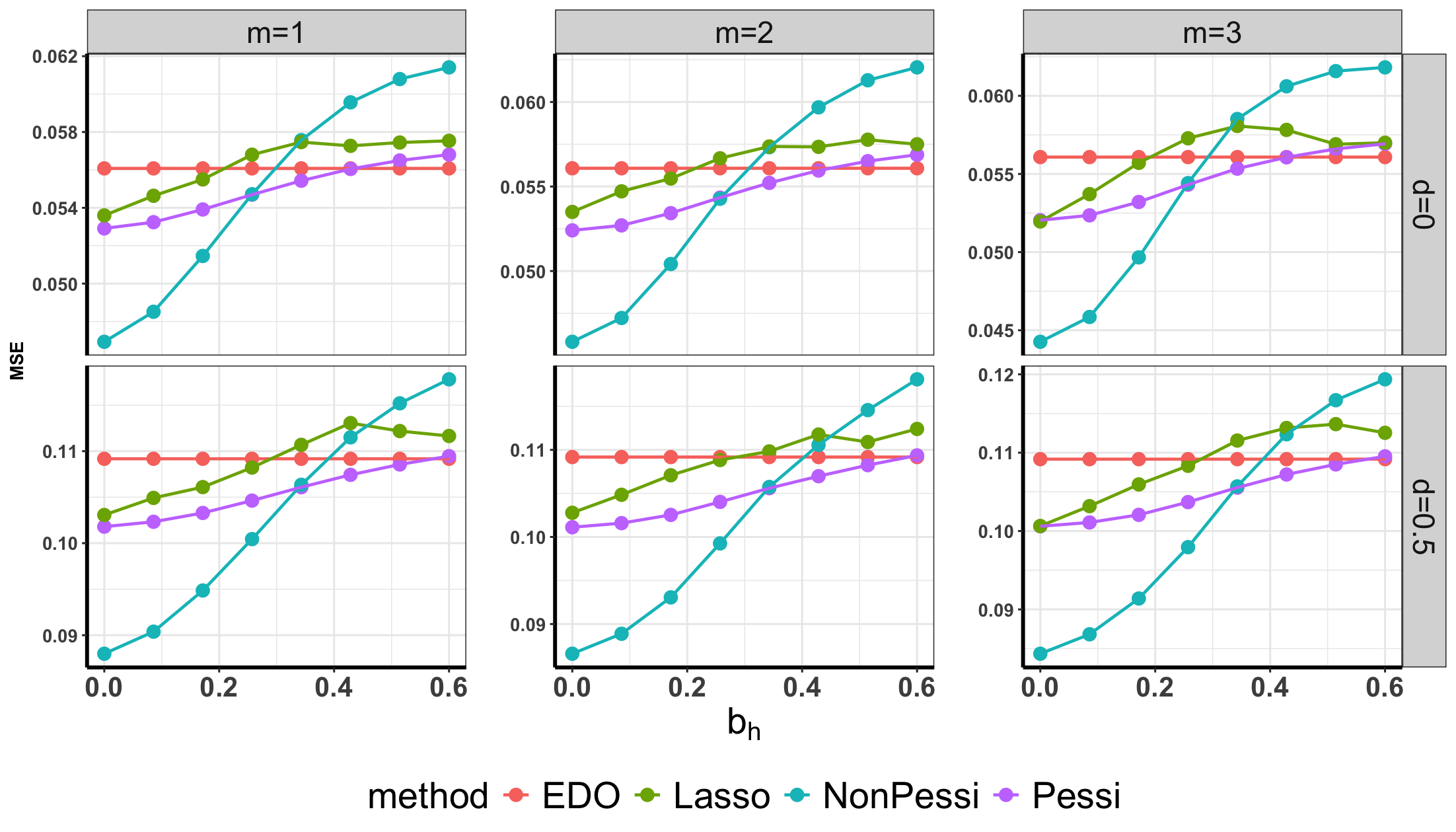}
            \includegraphics[height=6cm, width=8.5cm]{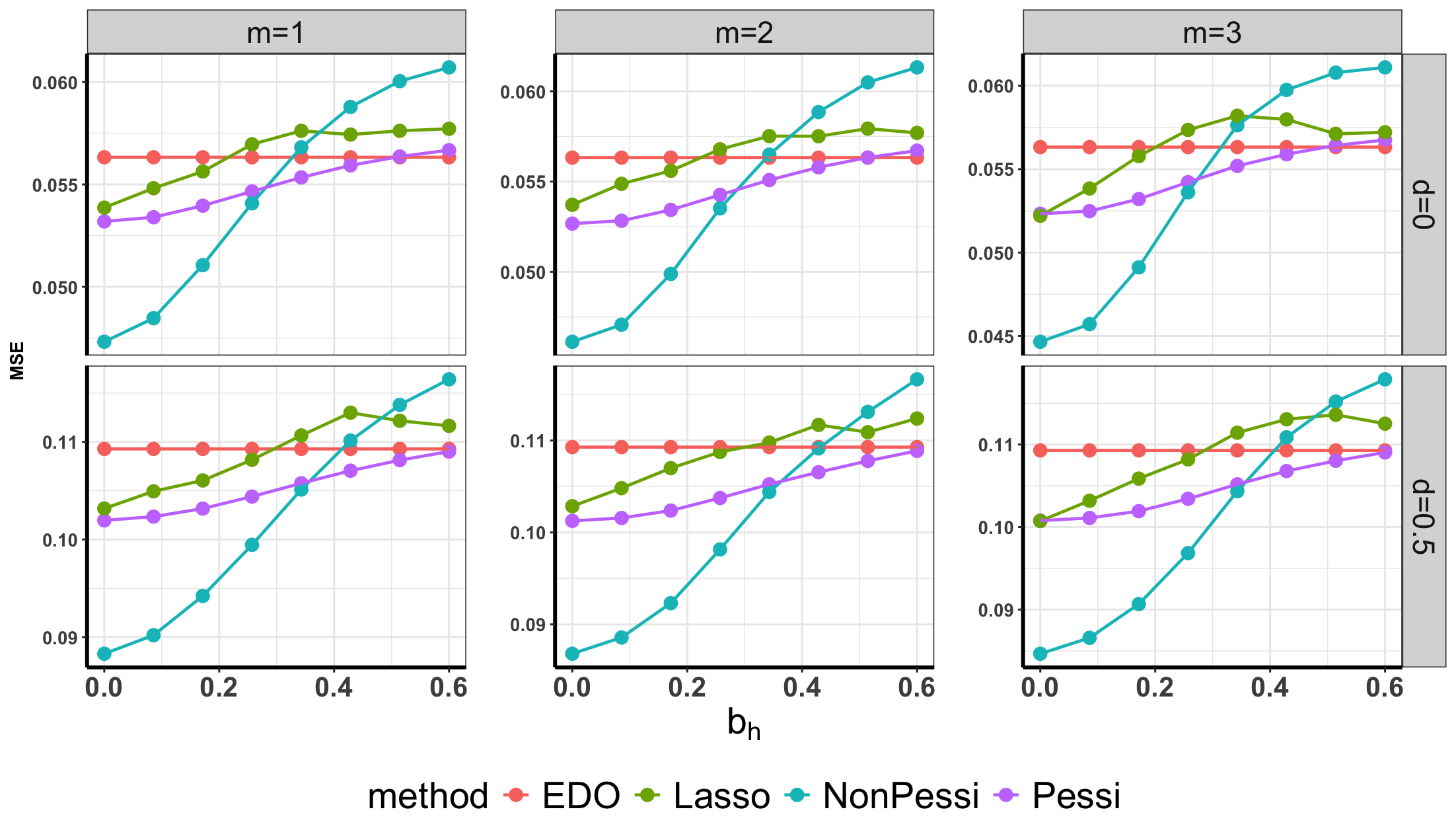}
	\caption{\small  Empirical means of MSEs for  various methods, excluding SPE,  under the switchback design in Example \ref{ex:real_data_based_non_dynamic}. The treatment effect ratio parameters are set to $5\%$ (Left) and $10\%$ (Right), respectively.}
	\label{fig:single_real_data_based_adaptive_all}
\end{figure}

{\textit{Example \ref{ex:clinical_data_based_non_dynamic}}} {\bf(Continued).}
We apply the proposed method to data from the AIDS Clinical Trials Group Protocol 175 (ACTG175), involving 2139 HIV-infected individuals. In this study, participants were randomly assigned to one of four treatment groups: zidovudine (ZDV) monotherapy, ZDV+didanosine (ddI), ZDV+zalcitabine, or ddI monotherapy \citep{hammer1996trial}. 
The CD4 count (cells/mm3
) at 20 
± 5 weeks post-baseline is chosen to be the continuous response variable $R$. 
Based on the significant factors identified in \citet{lu2013variable}, \citet{fan2017concordance} and \citet{shi2019testing}, 
the state variables are chosen to be age (in years), homosexual activity (0=no, 1=yes), and hemophilia (0=no, 1=yes).

We focus on two groups of patients receiving treatments ZDV+ddI (with a sample size of 522) and ZDV+zal (with a sample size of 524). 
To construct a simulation environment of these data, we 
use various nonlinear models to fit the data, taking into account the complex relationships between the state variables and the response. The goodness-of-fit for each model is assessed through examining their residual plots. This rigorous evaluation process has led us to select the following nonlinear model:
\begin{equation}\label{eq:CD4_model}
  \begin{split}
          R=f(S_1,S_2,S_3)+\gamma A S_1+\epsilon,
  \end{split}
\end{equation}
where $f(S_1,S_2,S_3)=(1+\beta_1 S_{1})^2+\beta_2 S_{2}+ \beta_3 S_{3}+ \beta_4 S_{1}S_{2}+\beta_5 S_{1}S_{3}$. In this model,
$R$ represents the CD4 count, $S_1, S_2, S_3$ correspond to age, homosexual activity, and hemophilia, and $A$ indicates whether a patient is receiving ZDV+ddI ($A=1$) or ZDV+zal ($A=0$). 

After obtaining the estimators of the unknown parameters $(\widehat \beta_1, \widehat \beta_2, \widehat \beta_3, \widehat \beta_4, \widehat \beta_5, \widehat \gamma)$, we calculate the fitted values $\widehat f(S_1,S_2,S_3)$ by plugging in the estimates along with the estimated residuals $\widehat \epsilon$.
This enables us to generate
the experimental data and the historical data similar to the bootstrap-assisted procedure in Algorithm \ref{algo:res_bootstrap} (Appendix A, page 15),
\begin{eqnarray*}
    R_e &=& \widehat f(S_{e,1},S_{e,2},S_{e,3}) +\widehat \gamma A_e S_{e,1}+ b_h + (1+d) \widehat \epsilon_e, \\
    R_h &=& \widehat f(S_{h,1},S_{h,2},S_{h,3}) + \widehat \epsilon_h,
\end{eqnarray*}
where the sets $(S_{e,1},S_{e,2},S_{e,3}, \widehat \epsilon_e)$ and $(S_{h,1},S_{h,2},S_{h,3}, \widehat \epsilon_h)$ are sampled from the state variables and the estimated residuals with replacement based on 
\eqref{eq:CD4_model}.
Furthermore, $A_e=1$ and $A_e=0$ are generated with equal probability, $b_h$ is varied from $0$ to $0.15$ in increments to produce a series of eight values, and $d\in \{0, 1\}$. The selected sample size is $|\mathcal{D}_e|=200$ and $|\mathcal{D}_h|=m |\mathcal{D}_e|$ with $m \in \{1,2,3\}$.

\begin{figure}[ht]
	\centering
      \includegraphics[height=5cm, width=8.5cm]{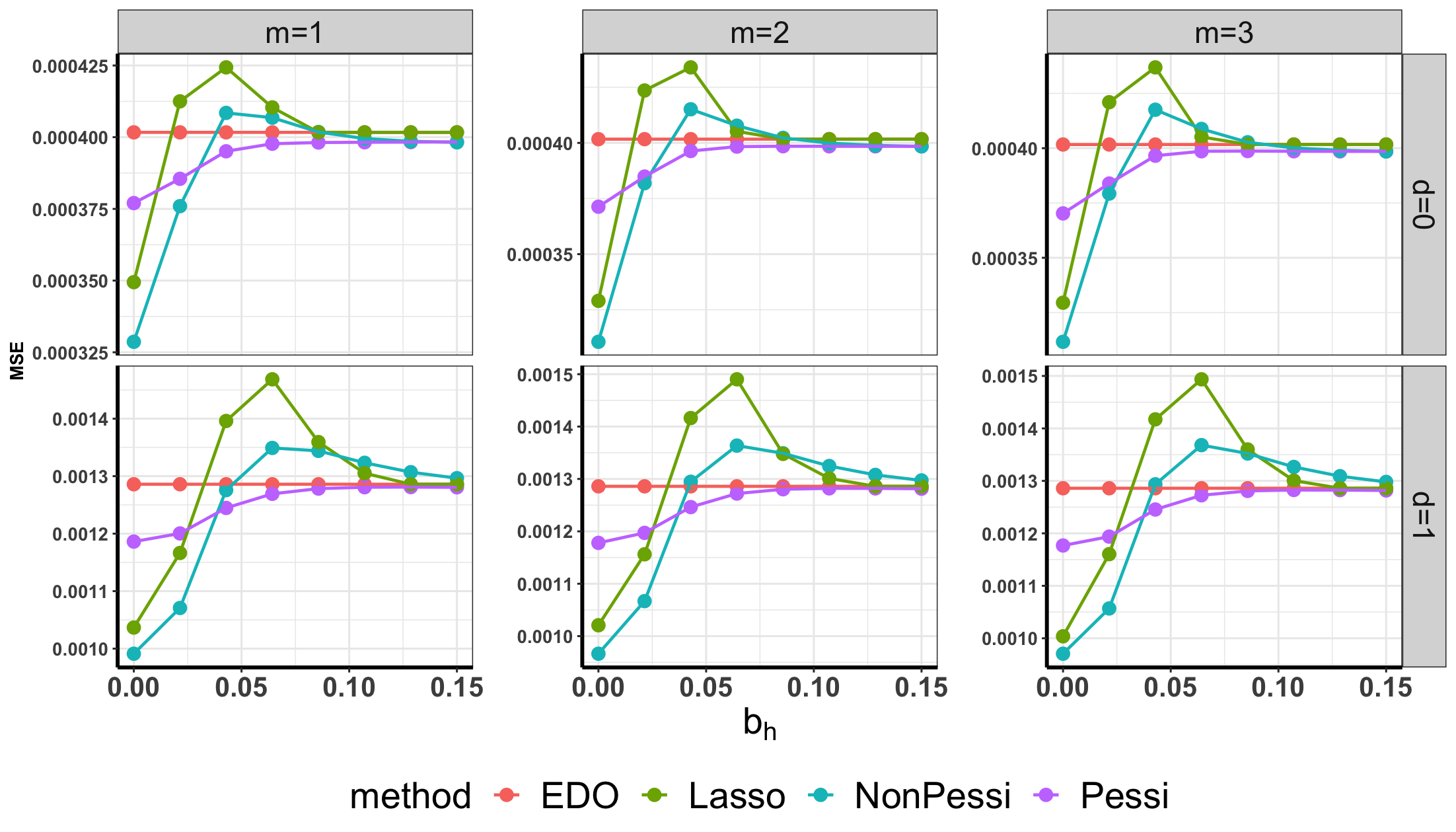}
            \includegraphics[height=5cm, width=8.5cm]{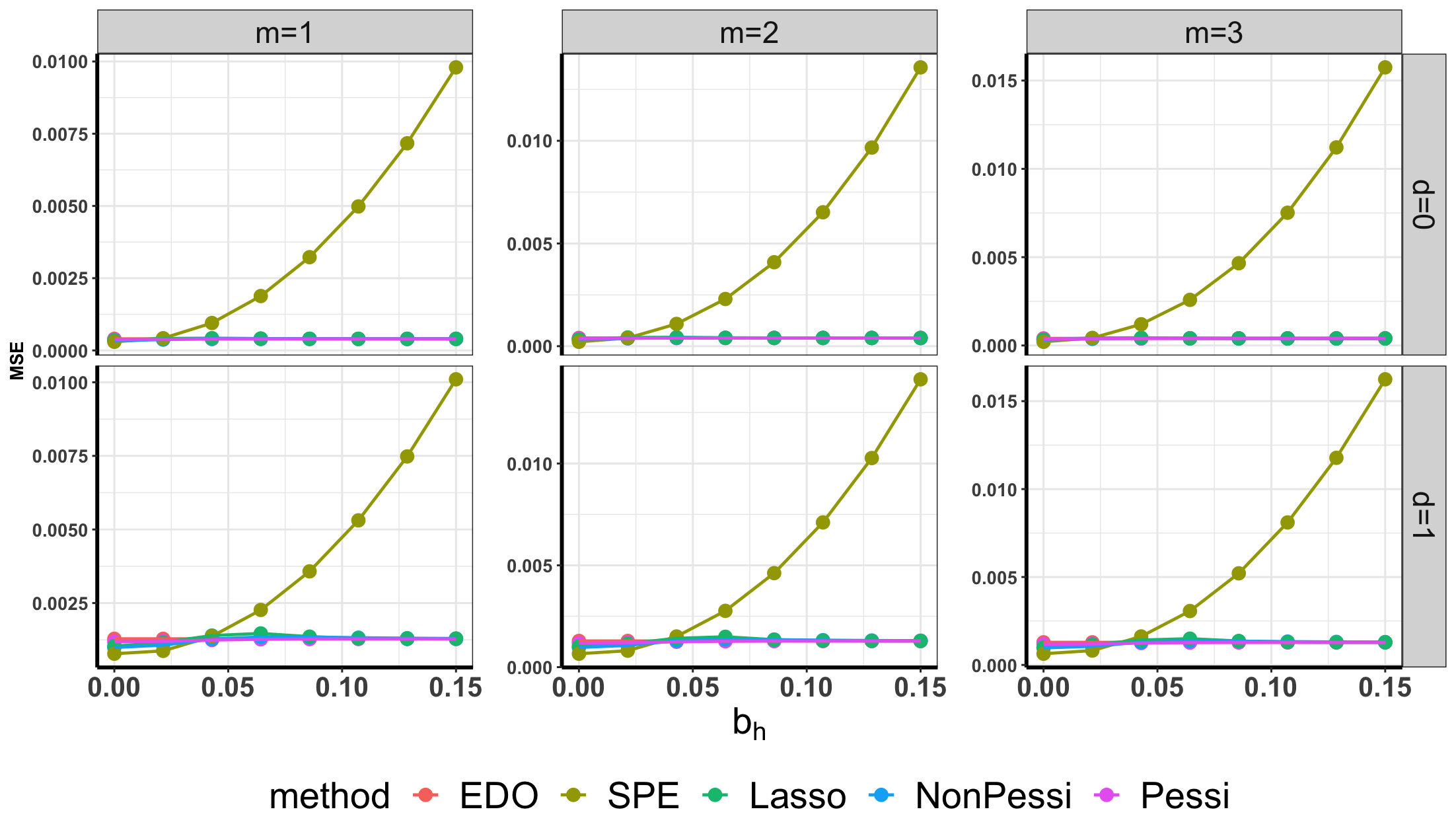}
	\caption{\small  Empirical means of MSEs for various methods in the clinical-data 
 based experiment in Example \ref{ex:clinical_data_based_non_dynamic}.}
	\label{fig:CD4_data_res}
\end{figure}
Figure \ref{fig:CD4_data_res} presents the empirical means of MSEs for all methods in the clinical-data based experiment. It further validates the effectiveness of the proposed methods, with findings aligning closely with those in Section \ref{sec:numerical}.

\begin{example}[Inference for the pessimistic method]
\label{ex:inference}
We further conduct an additional experiment to evaluate the coverage probabilities of the confidence intervals (CIs) based on our pessimistic estimator and compare them against the CIs based on the EDO and Lasso estimators.
Data settings mirror that in Example \ref{ex:single_stage}, except that $m=2$ is fixed and $|\mathcal{D}_e|=100$. The left panel of Figure \ref{fig:CI_results} displays the coverage probability of the $5\%$ confidence intervals across the three methods. 
Our findings indicate that the CIs based on both the pessimistic and EDO estimators achieve the expected nominal coverage (e.g., 0.95). In contrast, the CI based on the Lasso estimator exhibits significant undercoverage. Additionally, while maintaining nominal coverage, the pessimistic estimator yields narrower confidence intervals compared to the EDO estimator, indicating an improvement in efficiency by incorporating historical data.

\begin{figure}[ht]
	\centering
      \includegraphics[height=8cm, width=12cm]{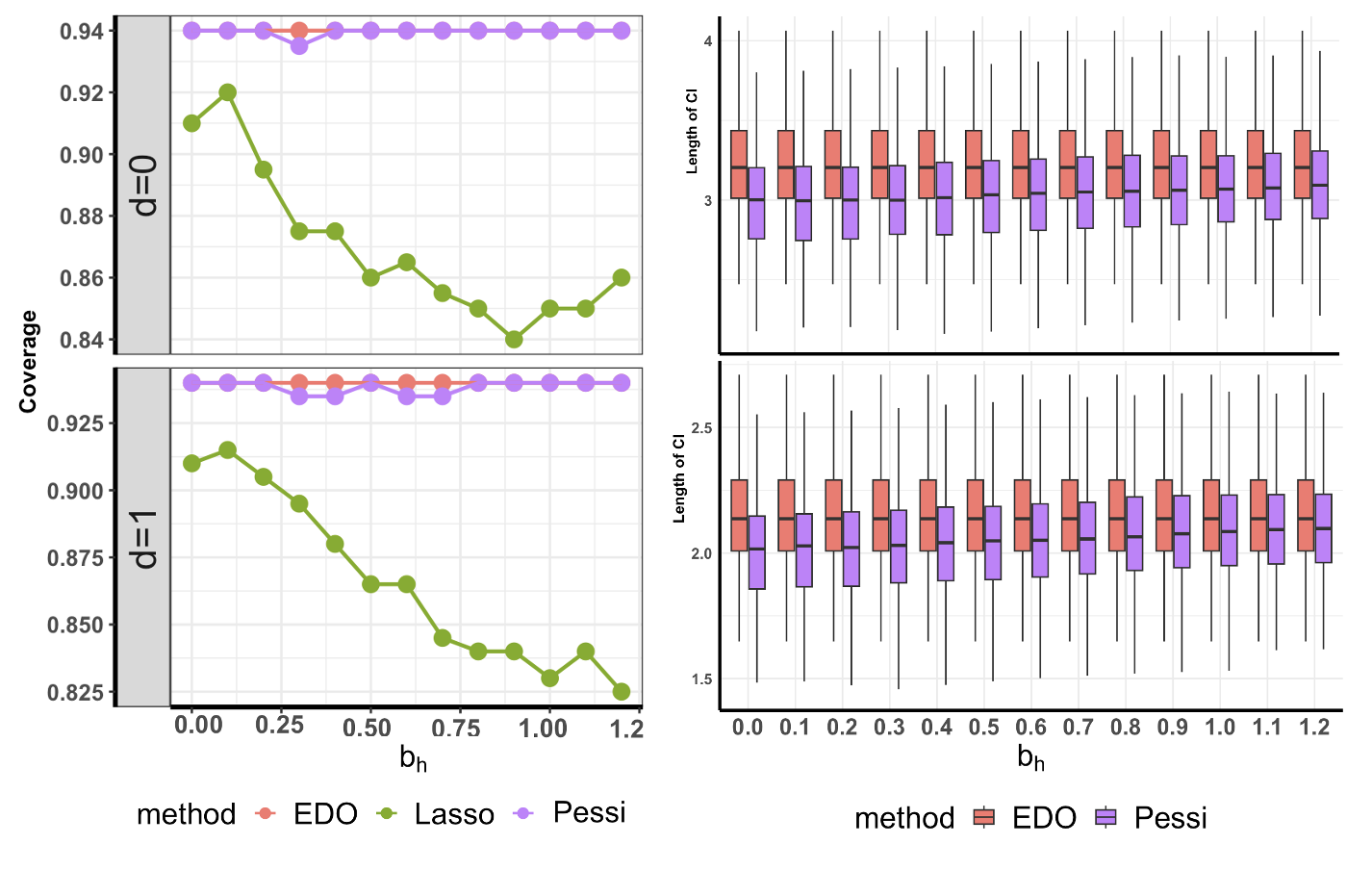}
	\caption{\small Coverage probabilities of the 95\% confidence interval (left panel) and the boxplots of the length of the confidence intervals (right panel).}
	\label{fig:CI_results}
\end{figure}
    
\end{example}

\section{A hybrid procedure}
\label{sec:hybrid_procedure}
 If we have prior knowledge about the reward shift $b_h$, we can introduce a hybrid procedure that leverages the strengths of each method within their optimal ranges. It consists of the following steps. 
For given thresholds $c_1, c_2$, 
\begin{itemize}
    \item 
    If $| b_h| \leq c_1 \sqrt{Var(\widehat b_h)} $, the bias is expected to be small and the SPE estimator \citep{li2023improving} is adopted.

    \item 
    If $c_1 \sqrt{Var(\widehat b_h) }<| b_h| \leq c_2 \sqrt{Var(\widehat b_h) }$, indicating moderate bias,  the proposed pessimistic method is then applied. 

    \item 
    If $| b_h| > c_2 \sqrt{Var(\widehat b_h)} $, signifying substantial bias and the EDO estimator $\widehat \tau_e$ is employed.
\end{itemize}
 According to our theoretical analysis, we can choose $c_1=1$ and $c_2=\sqrt{\log(n_{\min})}$. 
This ensures that scenarios, where variance dominates the bias, are categorized within the small reward shift region. Conversely, when the bias exceeds the established high confidence bound, it is classified under the large reward shift regime.

We conduct an additional experiment that includes the hybrid procedure, under the assumption that there is pre-existing knowledge about the specific regime to which the data corresponds. The experimental parameters are aligned with those outlined in Example \ref{ex:single_stage}, as detailed on page 7. Figure \ref{fig:hybrid_results} presents the empirical means of MSEs of all methods, with dotted vertical lines depicting the boundaries. These empirical results formally verify our theoretical assertions, demonstrating the superiority of the optimal method in each identified region. Additionally, the findings demonstrate that the hybrid procedure, when applied with prior knowledge of the data regime, consistently outperforms the rest of the methods. 

\begin{figure}[ht]
	\centering
      \includegraphics[height=5cm, width=12cm]{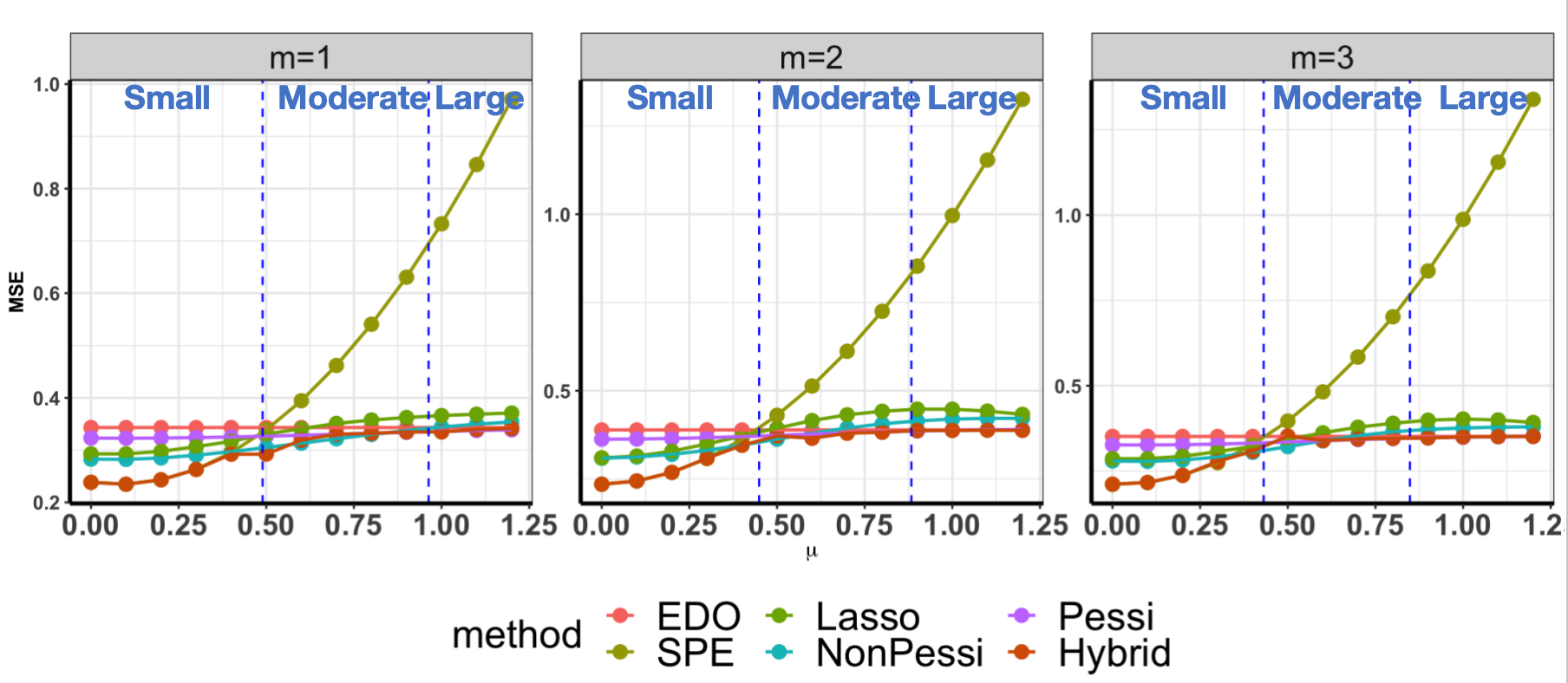}
	\caption{\small Empirical means of the MSEs across all the methods including the hybrid procedure. The vertical dotted blue lines denote the boundaries.}
	\label{fig:hybrid_results}
\end{figure}

In the absence of prior knowledge regarding $b_h$, estimating the regime to which the data belongs becomes necessary to implement the hybrid method. However, this estimation introduces additional variability. Consequently, our analysis reveals that the hybrid method, under these conditions, does not consistently outperform other methods in all scenarios. While optimizing the hybrid method for scenarios lacking prior regime knowledge presents a valuable avenue for future exploration, such an endeavor falls outside the scope of this current paper. 

\section{Extension to Sequential Decision Making}
\label{sec:theoretical_sequential}

Let $O_{e,t}$ be a shorthand of $(S_{e,t}, A_{e,t},R_{e,t})$ and denote $\{O_{e,t}\} $ containing all the data from $t=1, \dots, T$ in the experimental data. The EDO estimator can be represented by
$$
\widehat \tau_e = 
\frac{1}{|\mathcal{D}_e|}\sum_{\{O_{e,t}\} \in \mathcal{D}_e}\psi_e(\{O_{e,t}\}),
$$
where $\psi_e(\{O_{e,t}\}) =     \sum_{a=0}^1 (-1)^{a-1} \{V^a_{e,1} (S_{e,1})+ \sum_{t=1}^T \mu_t^a(A_{e,t},S_{e,t}) [R_{e,t}+ V^a_{e,t+1}(S_{e,t+1}) - V^a_{e,t}(S_{e,t}) ] \} $,
$\mu_t^1(A_{e,t},S_{e,t})$ and $\mu_t^0(A_{e,t},S_{e,t})$ are, respectively, the density ratio of the state-action pair of time $t$ under the treatment and control policy under the behavior policy of the experimental data, and $V^a_{e,t}(s)= \sum_{k=t}^T \Mean( R_{e,k} |S_{e,t}=s) $ is the value function for the experimental data. The estimator based on the historical data can be constructed as 
\begin{eqnarray*}
    \widehat{\tau}_h=\frac{1}{|\mathcal{D}_e|}\sum_{O_{e,1}\in \mathcal{D}_e}\psi_{h,1}(O_{e,1})-\frac{1}{|\mathcal{D}_h|}\sum_{\{O_{h,t}\}\in \mathcal{D}_h}\psi_{h,2}(\{O_{h,t}\}),
\end{eqnarray*}
where $\psi_{h,1}(O_{e,1}) = V^1_{e,1}(S_{e,1})+\sum_{t=1}^T\mu_t^1(A_{e,t}, S_{e,t}) [R_{e,t}+ V^1_{e,t+1}(S_{e,t+1}) - V^1_{e,t}(S_{e,t})]-V_{h,1}(S_e),$
$\psi_{h,2}(\{O_{h,t}\}) = \mu_t^h(S_{h,t})[ R_{h,t}+ V_{h,t+1}(S_{h,t+1}) - V_{h,t}(S_{h,t}) ]$, and 
$V_{h,t}(\bullet)= \sum_{k=t}^T \Mean( R_{h,k} |S_{h,t}=\bullet) $ is the value function for the historical data. 
The proposed estimator can be represented by $\widehat \tau_w = w \widehat \tau_e + (1 - w) \widehat \tau_h$.
In sequential decision-making, the bias caused by the reward shift is given by $b_h =\Mean[ V_{e,1}^0(S_e) ] -  \Mean[ V_{h,1} (S_e)  ]$.

The construction of $\widehat \tau_e $ and $\widehat \tau_w $ for the sequential decision making have similar patterns as that in Section \ref{sec:contextagnosticest}.
Following the methodologies developed in Section \ref{sec:contextagnosticest}, we adopt similar pessimistic and non-pessimistic strategies to determine the weight.

We next give the theoretical properties of the proposed estimators for sequential decision making. Similar to the analysis in the non-dynamic setting, we compare the proposed estimator with the oracle estimator $\widehat \tau_{w^*}$. Before that, we impose the following conditions that extend the Assumptions \ref{con:coverage}-\ref{con:double} to the sequential setting.

\begin{cond}
    \label{con:coverage sequential}
    There exists some constant $\varepsilon >0$ such that the true density ratios $\mu_t^{a,*}(a, s)\geq \epsilon$ and  $\mu_t^{h,*}(s)\geq \epsilon $ for any $t, a, s$.  
\end{cond}

\begin{cond}
\label{con:boundedness sequential}
(i) There exists some constant $R_{\max}$ such that $| R_{e,t}| \leq R_{\max}$ and $| R_{h,t}| \leq R_{\max}$ for $1\leq t \leq T$ almost surely. (ii) $| V_{e,t}^{a} |$ and $V_{h,t}$ are bounded by $(T+1-t)R_{\max}$.  (iii) 
$\mu_t^{a}$ and $\mu_t^h$ are lower bounded by $\varepsilon$.
\end{cond}

\begin{cond}
\label{cond:dr specification sequantial}
(i) Either $\mu_t^{a}$ or $V_{e,t}^{a}$ is correctly specified.  (ii) Either $\mu_t^h$ or $V_{h,t}$ is correctly specified.
\end{cond}

The following theorem provides a non-asymptotic upper bound for MSE of the non-pessimistic estimator.

\begin{theorem}[MSE of the non-pessimistic estimator]
\label{thm:MSE_nonpessi_sequential}
Under Assumptions \ref{con:coverage sequential}-\ref{cond:dr specification sequantial}, the excess MSE of the non-pessimistic estimator in sequential decision making satisfies
\begin{eqnarray*}
    \textrm{MSE}(\widehat \tau_{\widehat w} ) - \textrm{MSE}( \widehat \tau_{w^*} ) 
   \leq 
   \Mean\Big[ (1 - w^*)^2 - (1- \widehat w)^2 \Big](\widehat b^2_h - b_h^2) 
   + O \Big( \frac{ T^2 R^2_{\max}}{\varepsilon^2 n_{\min}^{3/2} } \Big).
\end{eqnarray*}
\end{theorem}
Compared to the MSE of the non-pessimistic estimator in Theorem \ref{thm:nonpSAE}, the first term is about the estimation error for the mean shift, and the second term is about the estimation errors for the variance and variance terms, which is inflated by a factor of $T^2$ due to the sequential setting.

\begin{corollary}[$b_h$ Conditions]
    \label{col:small bias sequential}
   (i)  If $|b_h|\ll n_{\min}^{-1/2} T R_{\max}/\epsilon$ and $\textrm{MSE}(\widehat \tau_{w^*} )$ is proportional to $T^2 R^2_{\max}/(\varepsilon^2 n_{\min})$, then
        \begin{eqnarray*}
        \Big|\frac{\textrm{MSE}(\widehat{\tau}_{\widehat{w}})-\textrm{MSE}(\widehat{\tau}_{w^*})}{\textrm{MSE}(\widehat{\tau}_{w^*})}-\frac{\textrm{SEE}(\widehat{b}_h)}{\textrm{MSE}(\widehat{\tau}_{w^*})}\Big|\to 0,
    \end{eqnarray*}
    as $n_{\min}\to \infty$, where $ \textrm{SEE}(\widehat{b}_h)=\Mean [(1-w^*)^2-(1-\widehat{w})^2] (\widehat{b}_h-b_h)^2$.
    
    (ii) If $|b_h | \gg n_{\min}^{-1/2} 
    \sqrt{\log( n_{\min})} T R_{\max}/\epsilon$, both 
    $\textrm{MSE}(\widehat{\tau}_{\widehat{w}})/\textrm{MSE}(\widehat{\tau}_{w^*}) $ and
    $\textrm{MSE}(\widehat{\tau}_{\widehat{w}})/\textrm{MSE}(\widehat{\tau}_{e}) $ approach 1 as $n_{\min} \rightarrow \infty$.
\end{corollary}

Part (i) of Corollary \ref{col:small bias sequential} gives the excess MSE against $\textrm{MSE}(\widehat{\tau}_{w^*})$ for a small $b_h$. Part (ii) of Corollary \ref{col:small bias sequential} shows the oracle property of the proposed non-pessimistic estimator for a large $b_h$.

\begin{theorem}[MSE of the pessimistic estimator]
    \label{thm:MSE_pessi_sequential}
 Under Assumptions \ref{con:coverage sequential}-\ref{cond:dr specification sequantial}, the excess MSE of the pessimistic estimator in sequential decision making satisfies
 \begin{eqnarray*}
 \textrm{MSE}(\widehat \tau_{\widehat w_U} ) - \textrm{MSE}( \widehat \tau_{w^*} ) 
   \leq 
    (1-w^*)^2 \Mean [(|\widehat{b}_h|+U)^2-b_h^2]+O\Big(\frac{T^2 R_{\max}^2}{\epsilon^2 n_{\min}^{3/2}}\Big)
    +O(\alpha [b_h^2+T^2 R_{\max}^2/\epsilon^2 n_{\min }]). 
\end{eqnarray*}
Furthermore, if $b_h\gg n_{\min}^{-1/2} (\log n_{\min})^{1/6} R_{\max} T /\epsilon$ and $U$ is proportional to the order $n_{\min}^{-1/2}\sqrt{\log n_{\min}} R_{\max}T/\epsilon$, then the pessimistic estimator achieves the oracle property.
\end{theorem}

Theorem \ref{thm:MSE_pessi_sequential} gives the excess MSE of the pessimistic estimator, which is also related to the estimation error of the bias, estimation errors of the variance and covariance terms, as well as the type I error $\alpha$. It further establishes the oracle property of the estimator for moderate and large $b_h$ cases.

The proofs of Theorem \ref{thm:MSE_nonpessi_sequential}, Corollary \ref{thm:MSE_pessi_sequential} and Theorem \ref{thm:MSE_pessi_sequential} are similar to proofs under the non-dynamic setting, we omit here.

\section{Implementation Details}\label{sec::implementation}
In this section, we present the construction of variance and covariance terms in \eqref{eqn:non-pess} and the estimation for sequence decision making.

{\bf Construction of variance and covariance terms in \eqref{eqn:non-pess}.}
We can obtain $\widehat \tau_e $ and $\widehat \tau_h $ by plugging in the unknown terms with their estimates 
\begin{eqnarray*}
   \widehat \tau_e = \frac{1}{|\mathcal{D}_e|}\sum_{O_e\in \mathcal{D}_e} \widehat \psi_e(O_e), \quad
    \widehat{\tau}_h =\frac{1}{|\mathcal{D}_e|}\sum_{O_e\in \mathcal{D}_e} \widehat \psi_{h,1}(O_e)-\frac{1}{|\mathcal{D}_h|}\sum_{O_h\in \mathcal{D}_h} \widehat \psi_{h,2}(O_h).
\end{eqnarray*}

Then we can obtain $\widehat{\Var}(\widehat{\tau}_e)$, $\widehat{\Var}(\widehat{\tau}_h)$, and $\widehat{\Cov}(\widehat{\tau}_e,\widehat{\tau}_h)$ 
by the corresponding sample variances
and sample covariance,
\begin{eqnarray*}
    \widehat{\Var}(\widehat{\tau}_e) &=& \frac{1}{|\mathcal{D}_e|} \sum_{O_e\in \mathcal{D}_e}(\widehat \psi_e(O_e)- \widehat \tau_e )^2, \\
    \widehat{\Var}(\widehat{\tau}_h) &=& \frac{1}{|\mathcal{D}_e|}\sum_{O_e\in \mathcal{D}_e} (\widehat \psi_{h,1}(O_e) - \prob_e \widehat \psi_{h,1}(O_e) )^2
    + \frac{1}{|\mathcal{D}_h|}\sum_{O_h\in \mathcal{D}_h} ( \widehat \psi_{h,2}(O_h)- \prob_h \widehat \psi_{h,2}(O_h)  )^2, \\
    \widehat{\Cov}(\widehat{\tau}_e,\widehat{\tau}_h) &=&  \frac{1}{|\mathcal{D}_e|}\sum_{O_e\in \mathcal{D}_e} (\widehat \psi_e(O_e)- \widehat \tau_e )(\widehat \psi_{h,1}(O_e) - \prob_e \widehat \psi_{h,1}(O_e) ),
\end{eqnarray*}
where $\prob_e \widehat \psi_{h,1}(O_e) = |\mathcal{D}_e|^{-1} \sum_{O_e\in \mathcal{D}_e}  \psi_{h,1}(O_e) $ and $\prob_h \widehat \psi_{h,2}(O_h) = |\mathcal{D}_h|^{-1} \sum_{O_h\in \mathcal{D}_h}  \psi_{h,2}(O_h) $.

{\bf Estimation for sequence decision making.} The estimator $\widehat \tau_e$ can be constructed using the existing estimation methods for doubly robust estimators, such as \cite{kallus2020double,jiang2016doubly} by plugging in the estimates $\{ \widehat V_{e,t}^a \}$ of the value function and estimates $\widehat \mu_{t}^{a}$ of the marginal density ratio.
The estimator $\widehat \tau_h$ can be constructed similarly by plugging the estimates
of the value function, and the estimates of the $\mu^h_t$. Next, we present the estimation of 
$\mu^h_t$.

Recall that $\mu^h_t (S_{t})$ is the ratio of the density of the state $S_{t}$
under the control policy in the experimental dataset to that in the historical dataset.
 It is noteworthy that for any function $f(S_{t})$, we have
\begin{eqnarray*}
  \Mean^h \Big[
{\mu}^h_t (S_{t}) f(S_{t})  \Big] 
= \Mean \Big[ {\mu}_{t}^{0} (A_{t-1}, S_{t-1}) f(S_{t}) { I(A_t=0) \over \pi(A_t|S_t)  } 
\Big],  
\end{eqnarray*}
where $\Mean^h$ is the expectation taken over the historical data, and $\pi(A_t|S_t)$ is the behavior policy that generates the experimental data.
Based on the above estimating equation, after we obtain ${\mu}_{t}^{0}$, ${\mu}^h_t (S_{t}) $
can be estimated recursively.
Specifically, we propose to approximate ${\mu}^h_t (S_{t}) $ using linear seievs \citep{chen2015optimal,chen2022well,bian2023off,shi2023dynamic} such that 
$\widehat{\mu}^h_t (S_t) = \Phi(S_t)^\top \widehat \gamma_t$ for some basis function $\Phi(S_t)$.
For $k=1, \dots, T$, taking $f(S_{k}) = \Phi(S_k)$,
$\widehat \gamma_k$ is determined by solving
 {\footnotesize
	\begin{align*}
		\frac{1}{|\mathcal{D}_h|}\sum_{\{O_{h,t}\}\in \mathcal{D}_h}  \sum_{t=k}^T \left( \Phi(S_{h,t})^\top \gamma_k  \Phi(S_{h,t}) \right) =
 \frac{1}{|\mathcal{D}_e|}\sum_{\{O_{e,t}\} \in \mathcal{D}_e} \sum_{t=k}^T \left( \widehat{\mu}_{t}^{0} (A_{e,t-1}, S_{e,t-1}) \Phi(S_{e,t}) { I(A_{e,t}=0)  \over\pi(A_{e,t}|S_{e,t}) } \right).   
\end{align*} }
If we have additional information that the distribution of the state variable under the control policy in the experimental data is the same as that of the state variable in the historical data, we can directly set $\mu_t^h(S_t)$=1 in the sequential setting and $\mu(S_t)=1$ in the non-dynamic setting.

\section{Proofs of the Theorems in Section \ref{sec:theorysampleestimator}}\label{sec:proofthm}
\subsection{Notations and Auxiliary Lemmas}
We begin by listing the notations used throughout the proof:
\begin{multicols}{2}
\begin{itemize}
    \item $\mathcal{D}_e$: the experimental data with size; 
    \item $\mathcal{D}_e^{(1)}$: the data subset to learn the weight; 
    \item $\mathcal{D}_e^{(2)}$: the data subset to construct ATE estimator; 
    \item $\mathcal{D}_h$: the historical data with size; 
    \item $\mathcal{D}_h^{(1)}$: the data subset to learn the weight;
    \item $\mathcal{D}_h^{(2)}$: the data subset to construct ATE estimator; 
    \item $O_e=(S_e, A_e, R_e)$: a context-action-reward triplet in $\mathcal{D}_e$;
    \item $O_h=(S_h,R_h)$: a context-reward pair in $\mathcal{D}_h$; 
    \item $\pi^*(a|s)$: the behavior policy that generates $\mathcal{D}_e$;
    \item $r_e^*(a,s)$: $\Mean (R_e|A_e=a, S_e=s)$;
    \item $\tau^*(s)$: $r_e(1,s)-r_e(0,s)$; 
    \item $r_h^*(s)$: $\Mean (R_h|S_h=s)$; 
    \item $r_e$, $r_h$, $\pi$ and $\mu$: the nuisance functions used to construct the doubly robust estimator;
    \item $\mu(s)$: the density ratio of the probability 
    \item $\widehat{\tau}_e$: the experimental-data-only estimator;
    \item $\psi_e$: the estimation function for $\widehat{\tau}_e$; 
    \item $\widehat{\tau}_h$: the estimator which incorporates the historical data;
    \item $\psi_{h,1}$ and $\psi_{h,2}$: the two estimation functions for $\widehat{\tau}_h$; 
    \item $b_h$: the mean shift $\Mean [r_e^*(0,S_e)-r_h^*(S_e)]$;
    \item $\widehat{b}_h$: the estimated bias;
    \item $\widehat{\Var},\widehat{\textrm{MSE}},\widehat{\Cov}$: the empirical variance, MSE and covariance estimator. 
\end{itemize}
\end{multicols}
We will sometimes write $\psi_e(O_e;\pi,r_e)$, $\psi_{h,1}(O_e;\pi,r_e)$ and $\psi_{h,2}(O_h;\mu,r_h)$ by $\psi_e(O_e)$, $\psi_{h,1}(O_e)$ and $\psi_{h,2}(O_h)$ when the nuisance functions used are clear from the context.

We next introduce the following auxiliary lemmas: 
\begin{lemma}\label{lemmaMSEw}
Under Assumption \ref{con:double}, for a given weight $w$, the MSE of the resulting doubly robust estimator $\widehat{\tau}_w$ is given by
\begin{eqnarray*}
    \frac{2w^2}{|D_e|} \Var(\psi_e(O_e;\pi,r_e))+\frac{2(1-w)^2}{|D_e|} \Var(\psi_{h,1}(O_e;\pi,r_e))+\frac{2(1-w)^2}{|D_h|} \Var(\psi_{h,2}(O_h;\mu,r_h))\\+\frac{4w(1-w)}{|D_e|}\Cov(\psi_e(O_e;\pi,r_e),\psi_{h,1}(O_e;\pi,r_e))+(1-w)^2 b_h^2.
\end{eqnarray*}
\end{lemma}
\begin{lemma}\label{lemmaMSEbound}
    Under Assumptions \ref{con:coverage}-\ref{con:double}, we have
    \begin{eqnarray*}
        &&\Mean |\widehat{\Var}(\psi_e(O_e;\pi,r_e))-\Var(\psi_e(O_e;\pi,r_e))|=O\Big(\frac{R_{\max}^2}{\sqrt{|D_e|}\epsilon^2}\Big),\\
        &&\Mean |\widehat{\Var}(\psi_{h,1}(O_e;\pi,r_e))-\Var(\psi_{h,1}(O_e;\pi,r_e))|=O\Big(\frac{R_{\max}^2}{\sqrt{|D_e|}\epsilon^2}\Big),\\
        &&\Mean |\widehat{\Var}(\psi_{h,2}(O_h;\pi,Q))-\Var(\psi_{h,2}(O_h;\mu,r_h))|=O\Big(\frac{R_{\max}^2}{\sqrt{|D_h|}\epsilon^2}\Big),\\
        &&\Mean |\widehat{\Cov}(\psi_e(O_e;\pi,r_e),\psi_{h,1}(O_e;\pi,r_e))-\Cov(\psi_e(O_e;\pi,r_e),\psi_{h,1}(O_e;\pi,r_e))|=O\Big(\frac{R_{\max}^2}{\sqrt{|D_e|}\epsilon^2}\Big).
    \end{eqnarray*}
\end{lemma}

\begin{lemma}
\label{lem:weight_error_bound}
Under conditions of Lemma \ref{lemmaMSEbound}, notice that $n_{\min}= \min\{|D_e|, |D_h| \}$, we have
\begin{eqnarray*}
\Mean | \widehat w  - w^* | &\leq & 
O \left( 
\frac{ R^2_{\max} n_{\min}^{-1/2} \epsilon^{-2} + R^2_{\max} \epsilon^{-2} 
+ |b_h| R^2_{\max} n_{\min}^{1/2} \epsilon^{-2} }{ R^2_{\max}  \epsilon^{-2} + n b_h^2 }
\right).
\end{eqnarray*}
\end{lemma}

\begin{proof}[Proof of Theorem \ref{thm:nonpSAE}]
    Since $\widehat{w}$ minimizes the empirical MSE, it follows that $\widehat{\text{MSE}}(\widehat{\tau}_{\widehat{w}})\le \widehat{\text{MSE}}(\widehat{\tau}_{w^*})$, which leads to 
    \begin{eqnarray*}
        \widehat{\text{MSE}}(\widehat{\tau}_{\widehat{w}})- \text{MSE}(\widehat{\tau}_{w^*})\le \widehat{\text{MSE}}(\widehat{\tau}_{w^*})-\text{MSE}(\widehat{\tau}_{w^*}).
    \end{eqnarray*}
    This allows us to upper bound the excess MSE $\text{MSE}(\widehat{\tau}_{\widehat{w}})-\text{MSE}(\widehat{\tau}_{w^*})$ by the sum of $\text{MSE}(\widehat{\tau}_{\widehat{w}})-\widehat{\text{MSE}}(\widehat{\tau}_{\widehat{w}})$ and $\widehat{\text{MSE}}(\widehat{\tau}_{w^*})-\text{MSE}(\widehat{\tau}_{w^*})$. 
    Let $\widetilde{\textrm{MSE}}$ denote a version of $\widehat{\textrm{MSE}}$ by replacing the sampling variance/covariance estimator with its oracle value. It follows from Lemmas \ref{lemmaMSEw} and \ref{lemmaMSEbound} that 
    \begin{eqnarray}\label{thm:proofnonSAEeq1}
    \begin{split}
        &\Mean [\widetilde{\textrm{MSE}}(\widehat{\tau}_{\widehat{w}})- \widehat{\textrm{MSE}}(\widehat{\tau}_{\widehat{w}})]\le \frac{2}{|D_e|}\Mean [|\widehat{\Var}(\psi_e(O_e))-\Var(\psi_e(O_e))|+|\widehat{\Var}(\psi_{h,1}(O_e))-\Var(\psi_{h,1}(O_e))|]\\
        +&\frac{2}{|D_h|}\Mean [|\widehat{\Var}(\psi_{h,2}(O_e))-\Var(\psi_{h,2}(O_e))|]+\frac{4}{|D_e|}\Mean |\widehat{\Cov}(\psi_e,\psi_{h,1})-\Cov(\psi_e,\psi_{h,1})|=O\Big(\frac{R_{\max}^2}{\epsilon^2 n_{\min}^{3/2} }\Big).
    \end{split}
    \end{eqnarray}
    Similarly, we can show that
    \begin{eqnarray}\label{thm:proofnonSAEeq2}
        \Mean [\widehat{\textrm{MSE}}(\widehat{\tau}_{w^*})-\widetilde{\textrm{MSE}}(\widehat{\tau}_{w^*})]=O\Big(\frac{R_{\max}^2}{\epsilon^2 n_{\min}^{3/2} }\Big).
    \end{eqnarray}
    Consequently, it remains to bound $\Mean [\text{MSE}(\widehat{\tau}_{\widehat{w}})-\widetilde{\text{MSE}}(\widehat{\tau}_{\widehat{w}})]$ and $\Mean [\widetilde{\text{MSE}}(\widehat{\tau}_{w^*})-\Mean \text{MSE}(\widehat{\tau}_{w^*})]$. By definition, we have
    \begin{eqnarray*}
        \Mean [\text{MSE}(\widehat{\tau}_{\widehat{w}})-\widetilde{\text{MSE}}(\widehat{\tau}_{\widehat{w}})] &=& \Mean (1-\widehat{w})^2 (b_h^2 - \widehat{b}_h^2),\\
        \Mean [\widetilde{\text{MSE}}(\widehat{\tau}_{w^*})-\text{MSE}(\widehat{\tau}_{w^*})] &=& \Mean (1-w^*)^2 (\widehat{b}_h^2 - b_h^2).
    \end{eqnarray*}
    The proof is hence completed. 
\end{proof}

\begin{proof}[Proof of Corollary \ref{cor:small_bh_nonpessimistic}.]
 We first prove the first part. According to \eqref{eqn:MSEdecomposition} and the definition of SEE in \eqref{eqn:sce}, one can derive that 
  \begin{eqnarray*}
  \Big|\textrm{MSE}(\widehat{\tau}_{\widehat{w}})-\textrm{MSE}(\widehat{\tau}_{w^*}) - \textrm{SEE}(\widehat{b}_h)\Big|
& \leq &
 \Big| 
 \Mean[ (1 - w^*)^2 - ( 1 - \widehat w)^2 ][(\widehat b_h^2 - b_h^2 ) - (\widehat b_h - b_h)^2 ] + O\Big(\frac{R_{\max}^2}{\epsilon^2 n_{\min}^{3/2} }\Big)
 \Big| \\
 & \leq &
 \Big| \Mean[ (1 - w^*)^2 - ( 1 - \widehat w)^2 ] [2 (\widehat b_h -b_h )b_h ] \Big| + O\Big(\frac{R_{\max}^2}{\epsilon^2 n_{\min}^{3/2} }\Big) \\
 & \ll &
 O\Big( {R_{\max}^2 \over \epsilon^2 n }  \Big) + O\Big(\frac{R_{\max}^2}{\epsilon^2 n_{\min}^{3/2} }\Big),
  \end{eqnarray*} 
where the last inequality follows from that  $\Mean|\widehat b_h -b_h |$ can be upper bounded by $O(n_{\min}^{-1/2} R_{\max} /\epsilon)$ and the condition $|b_h| \ll n_{\min}^{-1/2} R_{\max} / \epsilon$. Since $\textrm{MSE}(\widehat{\tau}_{w^*})$ is proportional to ${R_{\max}^2 /\epsilon^2 n_{\min} } $, it is easy to see that
\begin{eqnarray*}
    \Big|{ \textrm{MSE}(\widehat{\tau}_{\widehat{w}})-\textrm{MSE}(\widehat{\tau}_{w^*}) - \textrm{SEE}(\widehat{b}_h)   
    \over
    \textrm{MSE}(\widehat{\tau}_{w^*})
    }\Big| \rightarrow 0
\end{eqnarray*}
as $n_{\min} \rightarrow 0$.

For the second part, as \citet{li2023improving} selected the weight to attain the efficiency bound,  
it is sufficient to prove that the weight proposed in \citet{li2023improving}, under the constant potential outcome mean assumption and the proportionality assumption,
is equivalent to the oracle weight $w^*$.
To see this, the weight for each $S_e$ proposed by \citet{li2023improving} has the form 
\begin{eqnarray*}
    w_{EB} (S_e) &=& \frac{ |D_e| \mu(S_e) \pi^*(0|S_e)   }{ |D_e| \mu(S_e) \pi^*(0|S_e) + |D_h| \Var(R_e |A_e=0, S_e) / \Var(R_h | S_h)  } \\
    &=&
    \frac{ |D_e|    }{ |D_e|  + |D_h| \Var(R_e |A_e=0, S_e) / \Var(R_h | S_h)\mu(S_e) \pi^*(0|S_e)  }.
\end{eqnarray*}
Meanwhile, under the condition that $r_h(\bullet)=r_e(0,\bullet)$, $\widehat \tau_w$ can be rewritten as
\begin{eqnarray*}
  \widehat \tau_w 
  &=&
  \frac{w}{|\mathcal{D}_e|}\sum_{O_e\in \mathcal{D}_e}\psi_e(O_e)
  +
  \frac{1-w}{|\mathcal{D}_e|}\sum_{O_e\in \mathcal{D}_e}\psi_{h,1}(O_e)-\frac{1-w}{|\mathcal{D}_h|}\sum_{O_h\in \mathcal{D}_h}\psi_{h,2}(O_h) \\
  &=&
  \frac{1}{|\mathcal{D}_e|}\sum_{O_e\in \mathcal{D}_e} \Big\{r_e(1,S_e)+\nu^1(A_e|S_e) [R_e-r_e(A_e,S_e)] - r_e(0,S_e)  \Big\} \\
  && -
   \frac{w}{|\mathcal{D}_e|}\sum_{O_e\in \mathcal{D}_e} \nu^0(A_e|S_e) [R_e-r_e(A_e,S_e)]
   - \frac{1-w}{|\mathcal{D}_h|}\sum_{O_h\in \mathcal{D}_h}\psi_{h,2}(O_h).
\end{eqnarray*}
Then the optimal $w^*$ is minimizing the following variance
$$
w^2 |\mathcal{D}_e|^{-1} \Mean[ \Var(R_e |A_e=0, S_e) / \pi^*(0|S_e) ]   +
(1-w)^2 |\mathcal{D}_e|^{-1} \Mean[\Var(R_h | S_h)\mu(S_e)].
$$
It follows directly
$$
w^* = \frac{ |D_e|    }{ |D_e|  + |D_h| \Mean[ \Var(R_e |A_e=0, S_e) / \pi^*(0|S_e) ]/ \Mean[\Var(R_h | S_h)\mu(S_e)]  }.
$$
It is easy to see that $w_{EB} (S_e) = w^*$ under the proportionality assumption. 
\end{proof}

\begin{proof}[Proof of Corollary \ref{cor:moderate_bh_nonPessimistic}.]
For $|b_h| \gg n_{\min}^{-1/2} \sqrt{ \log n_{\min}  } R_{\max} / \varepsilon $, we can have
\begin{eqnarray*}
 w^*  = {  \Var(\widehat{\tau}_h) + b_h^2 -\Cov(\widehat{\tau}_e,\widehat{\tau}_h)  \over  \Var(\widehat{\tau}_e) + \Var(\widehat{\tau}_h) + b_h^2-2\Cov(\widehat{\tau}_e,\widehat{\tau}_h)  } \to 1
\end{eqnarray*}
as $n_{\min} \rightarrow \infty$.
This implies that $\text{MSE}(\widehat \tau_ {w^*}) / \text{MSE}(\widehat \tau_e) \rightarrow 1$ as $n_{\min} \rightarrow \infty$.
It suffices to show that $\widehat w \overset{p}{\to} 1$ as $n_{\min} \rightarrow \infty$. 

From Lemma \ref{lem:weight_error_bound}, we can see that if $|b_h| \gg n_{\min}^{-1/2} \sqrt{\log(n_{\min})}R_{\max} /\epsilon $, then 
\begin{eqnarray*}
  \Mean|\widehat w - w^*| \ll O( R_{\max} \epsilon / \sqrt{\log(n_{\min})} ).
\end{eqnarray*}
Therefore, it is easy to deduce that $\widehat w  \overset{p}{\to} w^*$ as $n_{\min} \to \infty$ by applying the Markov inequality.
This completes the proof.
\end{proof}

\begin{proof}[Proof of Theorem \ref{thm:pessi}.]
Let $\widebar{\textrm{MSE}}$ denote a version of $\widehat{\textrm{MSE}}$ by replacing the bias with its oracle value.
The difference between the MSE under $\widehat w_{\text{U}}$ and that under $w^*$ can be decomposed into
\begin{eqnarray}
\label{eq:MSE_decomposition}
     \text{MSE}(\widehat w_{\text{U}}) - \text{MSE}(w^*) 
   =
   \text{MSE}(\widehat w) - \widebar{\text{MSE}} (\widehat w)  + \widebar{\text{MSE}} (\widehat w) 
   - \widebar{\text{MSE}} ( w^* )  + \widebar{\text{MSE}} (w^*) -  \text{MSE}(w^*).   
   \end{eqnarray}
 Similar to the derivations  of \eqref{thm:proofnonSAEeq2}, we can obtain
\begin{eqnarray*}
        \Mean [\text{MSE}(\widehat w) - \widebar{\text{MSE}} (\widehat w)]=O\Big(\frac{R_{\max}^2}{\epsilon^2 n_{\min}^{3/2} }\Big),  \quad
         \Mean [\widebar{\text{MSE}} (\widehat w) 
   - \widebar{\text{MSE}} ( w^* ) ]=O\Big(\frac{R_{\max}^2}{\epsilon^2 n_{\min}^{3/2} }\Big).
\end{eqnarray*} 
Denote 
 $U(w) = (1-w)^2 [( |\widehat b_h | +  U)^2 -  b_h^2  ]$. We have that with probability $1-\alpha$,
\begin{eqnarray}
\label{eq:U_w_*}
\widebar{\text{MSE}} (\widehat w) 
   - \widebar{\text{MSE}} ( w^* ) 
    \leq 
    \widehat{\text{MSE}}_{\text{UB}}(\widehat w)   - \widebar{\text{MSE}}(w^*) 
     \leq 
    \widehat{\text{MSE}}_{\text{UB}}( w^* )   - \widebar{\text{MSE}}(w^*)\leq  U(w^*),
\end{eqnarray}
where the first inequality follows from the definition of $U$, and
the second inequality follows from the minimization of $\widehat{\text{MSE}}_{\text{UB}}$ at $\widehat w$. Hence, with probability $1 - \alpha$,
\begin{eqnarray*}
 \text{MSE}(\widehat w_{\text{U}}) - \text{MSE}(w^*)  
\leq 
\Mean( (1-w^*)^2 [( |\widehat b_h | +  U)^2 -  b_h^2  ] ) + O\Big(\frac{R_{\max}^2}{\epsilon^2 n_{\min}^{3/2} }\Big).
\end{eqnarray*}
According to \eqref{eqn:tauw} and Lemma \ref{lemmaMSEbound}
Meanwhile, the order of $  \text{MSE}(\widehat w_{\text{U}})  $ and $\text{MSE}(w^*)  $ can be given by $O([b_h^2+R_{\max}^2/\epsilon^2 n_{\min}])$.

Therefore, by decomposing  
\begin{eqnarray*}
   \text{MSE}(\widehat w_{\text{U}}) - \text{MSE}(w^*) 
    = 
  \text{MSE}(\widehat w_{\text{U}}) - \text{MSE}(w^*) \mathbb{I}( b_h^2 \leq ( |\widehat b_h|+  U )^2 )  +    \text{MSE}(\widehat w_{\text{U}}) - \text{MSE}(w^*) \mathbb{I}( b_h^2 > ( |\widehat b_h|+  U )^2 ) ,
\end{eqnarray*}
we can directly have
\begin{eqnarray*}
    \text{MSE}(\widehat w_{\text{U}}) - \text{MSE}(w^*) 
    \leq 
\Mean( (1-w^*)^2 [( |\widehat b_h | +  U)^2 -  b_h^2  ] ) + O\Big(\frac{R_{\max}^2}{\epsilon^2 n_{\min}^{3/2} }\Big) + O( \alpha[b_h^2+R_{\max}^2/\epsilon^2 n_{\min}] ).
\end{eqnarray*}
This completes the proof.
\end{proof}

\begin{proof}[Proof of Corollary \ref{coro:pessimistic}.]
According to the definition of $w^*$, we can have $1 - w^* = O( n^{-1} R^2_{\max} \epsilon^{-2} /(b_h^2+n^{-1} R^2_{\max} \epsilon^{-2} )$.
Notice that $\textrm{MSE}(\widehat{\tau}_{w^*}) = O( b_h^2+ n^{-1} R^2_{\max} \epsilon^{-2})$ and with probability $1-\alpha$, 
 $\Mean [(|\widehat{b}_h|+U)^2-b_h^2] = O(U^2 )$ according to \eqref{eqn:quantifier},
 if $b_h \gg n_{\min}^{-1/2}\log^{1/3} n_{\min} R_{\max}/\epsilon$,
 then $1-w^* =O( n^{-1} b_h^{-2})$, $\textrm{MSE}(\widehat{\tau}_{w^*}) = O(b_h^2)$ and 
\begin{eqnarray*}
  { (1-w^*)^2  \Mean[ ( |\widehat b_h | +  U)^2 -  b_h^2   ] \over \textrm{MSE}(\widehat{\tau}_{w^*}) }
    =  O\Big( {  U^2 \over n^2 b_h^{6} } \Big) \rightarrow 0
    \end{eqnarray*}
    as $n_{\min} \rightarrow \infty$ under the condition $U$ is proportional to the order $n_{\min}^{-1/2}\sqrt{\log n_{\min}} R_{\max}/\epsilon$ and $b_h\gg n_{\min}^{-1/2} (\log n_{\min})^{1/6} R_{\max}/\epsilon$.

    Meanwhile, $\alpha [b_h^2+R_{\max}^2/\epsilon^2 n_{\min}] / \textrm{MSE}(\widehat{\tau}_{w^*}) = o( n^{-1}_{\min} )$,
    $\alpha [b_h^2+R_{\max}^2/\epsilon^2 n_{\min}] /\textrm{MSE}(\widehat{\tau}_{w^*}) \rightarrow 0$ as $n_{\min} \rightarrow \infty$.
    It leads to the following
    \begin{eqnarray*}
      {\textrm{MSE}(\widehat{\tau}_{\widehat{w}_U})-\textrm{MSE}(\widehat{\tau}_{w^*}) \over \textrm{MSE}(\widehat{\tau}_{w^*}) } \rightarrow 0.
    \end{eqnarray*}
  This completes the proof.
\end{proof}

\section{Auxiliary Lemmas}
In this section, we provide the proofs of lemmas used in Section \ref{sec:proofthm}. 
\begin{proof}[Proof of Lemma \ref{lemmaMSEw}]
    The proof of Lemma \ref{lemmaMSEw} is straightforward. The key is to observe that both $\widehat{\tau}_e$ and $\widehat{\tau}_h$ are unbiased under Assumption \ref{con:double}, due to the doubly robust property. This allows us to decompose the MSE into the sum of the squared bias term $b_h^2$ and the variance term $w^2 \Var(\widehat{\tau}_e)+(1-w)^2 \Var(\widehat{\tau}_h)+2w(1-w)\Cov(\widehat{\tau}_e,\widehat{\tau}_h)$. Additionally, the variance terms $\Var(\widehat{\tau}_h)$ and $\Cov(\widehat{\tau}_e,\widehat{\tau}_h)$ can be further simplified due to the independence between the experimental and historical datasets. This leads to the desired conclusion. 
\end{proof}

\begin{proof}[Proof of Lemma \ref{lemmaMSEbound}]
    We focus on bounding the expected absolute difference between $\Var(\psi_e(O_e;\pi,r_e))$ and its sampling variance estimator $\widehat{\Var}(\psi_e(O_e;\pi,r_e))$ in this section. Other bounds can be similarly established. 
    
    Without loss of generality, we assume the sampling variance formula scales with $1/n$ instead of $1/(n-1)$. The error bounds are asymptotically the same when the formula scales with $1/(n-1)$. By definition, we can upper bound $\Mean |\widehat{\Var}(\psi_e(O_e;\pi,r_e))-\Var(\psi_e(O_e;\pi,r_e))|$ into the sum of $\Mean |\widehat{\tau}^2_e-\tau^2|$ and 
    \begin{eqnarray}\label{eqnproofbound1}
        \Mean \Big|\frac{2}{|D_e|}\sum_{O_e\in \mathcal{D}_{e,2}}\psi_e^2(O_e;\pi,Q)-\Mean \psi_e^2(O_e;\pi,Q)\Big|.
    \end{eqnarray}
    With some calculations, we have
    \begin{eqnarray}
    \label{eq:tau_e_square_error_bound}
        \Mean |\widehat{\tau}^2_e-\tau^2|=\Mean |\widehat{\tau}_e-\tau| |\widehat{\tau}_e+\tau|\le \Mean |\widehat{\tau}_e-\tau|^2+2\tau \Mean |\widehat{\tau}_e-\tau|\le \Mean |\widehat{\tau}_e-\tau|^2+2\tau \sqrt{\Mean |\widehat{\tau}_e-\tau|^2}. 
    \end{eqnarray}
    According to Lemma \ref{lemmaMSEw}, the MSE of $\widehat{\tau}_e$ equals $2\Var(\psi_e(O_e;\pi,r_e))/|D_e|$ which can be upper bounded by $O(R_{\max}^2/(|D_e|\epsilon^2))$ under Assumptions \ref{con:coverage} and \ref{con:bound}. Additionally, notice that $\tau$ can be upper bounded by $R_{\max}$. This allows us to bound $\Mean |\widehat{\tau}^2_e-\tau^2|$ by $O(R_{\max}^2/(|D_e|\epsilon^2))+O(R_{\max}/\sqrt{|D_e|}\epsilon)$.

    We next upper bound \eqref{eqnproofbound1}. An application of Cauchy-Schwarz inequality yields that \eqref{eqnproofbound1} can be upper bounded by $\sqrt{\Var(2\sum_{O_e\in \mathcal{D}_{e,2}}\psi_e^2(O_e;\pi,Q)/|D_e|)}=2\sqrt{\Var(\psi_e^2(O_e;\pi,Q))/|D_e|}=O(R_{\max}^2/\sqrt{|D_e|}\epsilon^2)$. Since $\epsilon\le 1$, the difference $\Mean |\widehat{\Var}(\psi_e(O_e;\pi,r_e))-\Var(\psi_e(O_e;\pi,r_e))|$ can be upper bound by $O(R_{\max}^2/\sqrt{|D_e|}\epsilon^2)$. 
\end{proof}

\begin{proof}[Proof of Lemma \ref{lem:weight_error_bound}.]
Denote $\widehat F_1= \widehat{b}_h^2+\widehat{\Var}(\widehat{\tau}_h)-\widehat{\Cov}(\widehat{\tau}_e,\widehat{\tau}_h)$,
$\widehat F_2 = \widehat{\Var}(\widehat{\tau}_e)+\widehat{b}_h^2+\widehat{\Var}(\widehat{\tau}_h)-2\widehat{\Cov}(\widehat{\tau}_e,\widehat{\tau}_h) $,
and $F_1= {b}_h^2+{\Var}(\widehat {\tau}_h)-{\Cov}(\widehat{\tau}_e,\widehat{\tau}_h)$,
$ F_2 ={\Var}(\widehat{\tau}_e)+{b}_h^2+{\Var}(\widehat{\tau}_h)-2{\Cov}(\widehat{\tau}_e,\widehat{\tau}_h)$.
According to the formulations of $\widehat w$ and $w^*$, and $F_1 < F_2$,  it is easy to derive that,
\begin{eqnarray*}
| \widehat w  - w^* |
=\left|{\widehat F_1 \over \widehat F_2} - { F_1 \over  F_2}  \right|
=\left| 
{ (F_1 - \widehat  F_1) F_2 + F_1 (  F_2 - \widehat F_2) \over \widehat F_2 F_2  }
\right|
\leq 
{ \left|  (\widehat F_1 - F_1)  \right| +  \left|  (  F_2 - \widehat F_2) \right| \over F_2   }
\left| 
{ F_2 \over \widehat F_2  }
\right|.
\end{eqnarray*}
It is easy to derive that
\begin{eqnarray*}
  | \widehat w  - w^* | &\leq &
 { |\widehat{b}_h^2 - b_h^2| + 2 |\widehat {\Cov}(\widehat{\tau}_e,\widehat{\tau}_h) - {\Cov}(\widehat{\tau}_e,\widehat{\tau}_h)| + |\widehat{\Var}(\widehat{\tau}_e) - {\Var}(\widehat{\tau}_e)| +  | \widehat{\Var}(\widehat{\tau}_h) - {\Var}(\widehat{\tau}_h)|
   \over {\Var}(\widehat{\tau}_e)+{b}_h^2+{\Var}(\widehat{\tau}_h)-2{\Cov}(\widehat{\tau}_e,\widehat{\tau}_h) }. 
\end{eqnarray*}
It follows from Lemma \ref{lemmaMSEbound} and the derivation of \eqref{eq:tau_e_square_error_bound} that
\begin{eqnarray*}
  \Mean |\widehat{b}_h^2 - b_h^2| \leq O\Big(\frac{R_{\max}^2}{\epsilon^2 |D_e| }\Big) + 
  O\Big(\frac{b_h R_{\max}}{\epsilon \sqrt{|D_e|} }\Big)
  \quad
  \Mean |\widehat{\Var}(\widehat{\tau}_e) - {\Var}(\widehat{\tau}_e)| \leq O\Big(\frac{R_{\max}^2}{\epsilon^2 |D_e|^{3/2} }\Big), \\
  \Mean |\widehat{\Var}(\widehat{\tau}_h) - {\Var}(\widehat{\tau}_h)| \leq O\Big(\frac{R_{\max}^2}{\epsilon^2 |D_h|^{3/2} }\Big), \quad
  \Mean |\widehat {\Cov}(\widehat{\tau}_e,\widehat{\tau}_h) - {\Cov}(\widehat{\tau}_e,\widehat{\tau}_h)| 
  \leq 
  O\Big(\frac{R_{\max}^2}{\epsilon^2 |D_e|^{3/2} }\Big).
\end{eqnarray*}
Then it is easy to derive that 
\begin{eqnarray*}
\Mean | \widehat w  - w^* | &\leq & 
O \left( 
\frac{ R^2_{\max} n_{\min}^{-1/2} \epsilon^{-2} + R^2_{\max} \epsilon^{-2} 
+ b_h R_{\max} n_{\min}^{1/2} \epsilon^{-1} }{ R^2_{\max}  \epsilon^{-2} + n_{\min} b_h^2 }
\right)
\end{eqnarray*}
This completes the proof.
\end{proof}

\end{document}